\newcommand*{\NEURIPS}{}

\newcommand*{\CAMREADY}{}

\ifdefined\NEURIPS
	\PassOptionsToPackage{numbers}{natbib}
	\documentclass{article}
	\ifdefined\CAMREADY	
		\usepackage[preprint]{neurips_2019}
	\else
		\usepackage{neurips_2019}
	\fi
	\usepackage[utf8]{inputenc} 
	\usepackage[T1]{fontenc}    
	\usepackage{hyperref}       	
	\usepackage{url}            	
	\usepackage{booktabs}       
	\usepackage{amsfonts}       
	\usepackage{nicefrac}       	
	\usepackage{microtype}      
\fi
\ifdefined\CVPR
	\documentclass[10pt,twocolumn,letterpaper]{article}
	\usepackage{cvpr}
	\usepackage{times}
	\usepackage{epsfig}
	\usepackage{graphicx}
	\usepackage{amsmath}
	\usepackage{amssymb}
	\usepackage[square,comma,numbers]{natbib}
\fi
\ifdefined\AISTATS
	\documentclass[twoside]{article}
	\ifdefined\CAMREADY
		\usepackage[accepted]{aistats2015}
	\else
		\usepackage{aistats2015}
	\fi
	\usepackage{url}
	\usepackage[round,authoryear]{natbib}
\fi
\ifdefined\ICML
	\documentclass{article}
	\usepackage{microtype}
	\usepackage{graphicx}
	\usepackage{subfigure}
	\usepackage{booktabs}
	\usepackage{hyperref}
	
	\ifdefined\CAMREADY
		\usepackage[accepted]{icml2018}
	\else
		\usepackage{icml2018}
	\fi
\fi
\ifdefined\ICLR
	\documentclass{article}
	\usepackage{iclr2019_conference,times}
	\usepackage{hyperref}
	\usepackage{url}

	\ifdefined\CAMREADY
		\iclrfinalcopy
	\fi
\fi
\ifdefined\COLT
	\ifdefined\CAMREADY
		\documentclass{colt2017}
	\else
		\documentclass[anon,12pt]{colt2017}
	\fi
	\usepackage{times}
\fi

\usepackage{color}
\usepackage{amsfonts}
\usepackage{algorithm}
\usepackage{algorithmic}
\usepackage{graphicx}
\usepackage{nicefrac}
\usepackage{bbm}
\usepackage{wrapfig}
\usepackage{tikz}
\usepackage[titletoc,title]{appendix}
\usepackage{stmaryrd}
\usepackage{comment}
\usepackage{endnotes}
\usepackage{hyperendnotes}
\usepackage{enumerate}
\ifdefined\COLT
\else
	\usepackage{amsmath}
	\usepackage{amsthm}
    \fi
\usepackage[small]{caption}
\usepackage[caption = false]{subfig}
\usepackage{footmisc}

\ifdefined\COLT
	\newtheorem{claim}[theorem]{Claim}
	\newtheorem{fact}[theorem]{Fact}
	\newtheorem{procedure}{Procedure}
	\newtheorem{conjecture}{Conjecture}	
	\newtheorem{hypothesis}{Hypothesis}	
	\newcommand{\qed}{\hfill\ensuremath{\blacksquare}}
\else
	
	\newtheorem{lemma}{Lemma}
	\newtheorem{corollary}{Corollary}
	\newtheorem{theorem}{Theorem}
	\newtheorem{proposition}{Proposition}

	\newtheorem{conjecture}{Conjecture}		
	
\fi

\def\be{\begin{equation}}
\def\ee{\end{equation}}
\def\beas{\begin{eqnarray*}}
\def\eeas{\end{eqnarray*}}
\def\bea{\begin{eqnarray}}
\def\eea{\end{eqnarray}}

\newcommand{\y}{{\mathbf y}}

\newcommand{\uu}{{\mathbf u}}
\newcommand{\vv}{{\mathbf v}}
\newcommand{\rr}{{\mathbf r}}

\newcommand{\s}{{\mathbf s}}

\newcommand{\bnu}{{\boldsymbol \nu}}

\newcommand{\A}{{\mathcal A}}

\newcommand{\D}{{\mathcal D}}

\renewcommand{\S}{{\mathcal S}}

\newcommand{\R}{{\mathbb R}}
\newcommand{\N}{{\mathbb N}}

\newcommand{\norm}[1]{\left\|#1 \right\|}

\newcommand{\inprod}[2]  {\left\langle{#1},{#2}\right\rangle}

\DeclareMathOperator*{\argmin}{argmin}

\definecolor{xcolor-gray}{gray}{0.95}

\newcommand{\rank}{\mathrm{rank}}
\newcommand{\shallow}{\mathrm{sha}}
\newcommand{\deep}{\mathrm{deep}}
\DeclareMathOperator{\diag}{diag}
\newcommand{\opt}{\mathsf{OPT}}

\definecolor{darkspringgreen}{rgb}{0.09, 0.45, 0.27}
\usetikzlibrary{decorations.pathreplacing,calligraphy}
\usetikzlibrary{patterns}

\ifdefined\CVPR
	\usepackage[breaklinks=true,bookmarks=false]{hyperref}
	\ifdefined\CAMREADY
		\cvprfinalcopy 
	\fi
	
\fi
\ifdefined\NOTESAPP
	\let\note\endnote
	\let\notetext\endnotetext
\else
	\let\note\footnote
	\let\notetext\footnotetext
\fi

\ifdefined\ICML
	\newcommand*{\ABBR}{}
\fi
\ifdefined\NEURIPS
	\newcommand*{\ABBR}{}
\fi
\ifdefined\ICLR
	\newcommand*{\ABBR}{}
\fi
\ifdefined\COLT
	\newcommand*{\ABBR}{}
\fi
\ifdefined\ABBR
	\newcommand{\eg}{{\it e.g.}}
	\newcommand{\ie}{{\it i.e.}}
	\newcommand{\cf}{{\it cf.}}
	
	\newcommand{\vs}{{\it vs.}}

	\newcommand{\etal}{{\it et al.}}
\fi

\begin{document}

\ifdefined\NEURIPS
	\title{Implicit Regularization in Deep Matrix Factorization}
	\author{
	\hspace{-2mm} Sanjeev Arora \\
	\hspace{-2mm} Princeton University and Institute for Advanced Study \\
	\hspace{-2mm} \texttt{arora@cs.princeton.edu} \\
	\And 
	\hspace{9mm} Nadav Cohen \\
	\hspace{9mm} Tel Aviv University \\
	\hspace{9mm} \texttt{cohennadav@cs.tau.ac.il} \\
	\And 
	\hspace{17mm} Wei Hu \\
	\hspace{17mm} Princeton University \\
	\hspace{17mm} \texttt{huwei@cs.princeton.edu} \\
	\And 
	\hspace{27mm} Yuping Luo \\
	\hspace{27mm} Princeton University \\
	\hspace{27mm} \texttt{yupingl@cs.princeton.edu} \\
	}
	\maketitle
\fi
\ifdefined\CVPR
	\title{Paper Title}
	\author{
	Author 1 \\
	Author 1 Institution \\	
	\texttt{author1@email} \\
	\and
	Author 2 \\
	Author 2 Institution \\
	\texttt{author2@email} \\	
	\and
	Author 3 \\
	Author 3 Institution \\
	\texttt{author3@email} \\
	}
	\maketitle
\fi
\ifdefined\AISTATS
	\twocolumn[
	\aistatstitle{Paper Title}
	\ifdefined\CAMREADY
		\aistatsauthor{Author 1 \And Author 2 \And Author 3}
		\aistatsaddress{Author 1 Institution \And Author 2 Institution \And Author 3 Institution}
	\else
		\aistatsauthor{Anonymous Author 1 \And Anonymous Author 2 \And Anonymous Author 3}
		\aistatsaddress{Unknown Institution 1 \And Unknown Institution 2 \And Unknown Institution 3}
	\fi
	]	
\fi
\ifdefined\ICML
	\twocolumn[
	\icmltitlerunning{Paper Title}
	\icmltitle{Paper Title} 
	\icmlsetsymbol{equal}{*}
	\begin{icmlauthorlist}
	\icmlauthor{Author 1}{institutionA} 
	\icmlauthor{Author 2}{institutionB}
	\icmlauthor{Author 3}{institutionA,institutionB}
	\end{icmlauthorlist}
	\icmlaffiliation{institutionA}{Department A, University A, City A, Region A, Country A}
	\icmlaffiliation{institutionB}{Department B, University B, City B, Region B, Country B}
	\icmlcorrespondingauthor{Corresponding Author 1}{cauthor1@email}
	\icmlcorrespondingauthor{Corresponding Author 2}{cauthor2@email}
	\icmlkeywords{Deep Learning, Learning Theory, Non-Convex Optimization}
	\vskip 0.3in
	]
	\printAffiliationsAndNotice{} 
\fi
\ifdefined\ICLR
	\title{Paper Title}
	\author{
	Author 1 \\
	Author 1 Institution \\
	\texttt{author1@email}
	\And
	Author 2 \\
	Author 2 Institution \\
	\texttt{author2@email}
	\And
	Author 3 \\ 
	Author 3 Institution \\
	\texttt{author3@email}
	}
	\maketitle
\fi
\ifdefined\COLT
	\title{Paper Title}
	\coltauthor{
	\Name{Author 1} \Email{author1@email} \\
	\addr Author 1 Institution
	\And
	\Name{Author 2} \Email{author2@email} \\
	\addr Author 2 Institution
	\And
	\Name{Author 3} \Email{author3@email} \\
	\addr Author 3 Institution}
	\maketitle
\fi

\begin{abstract}

Efforts to understand the generalization mystery in deep learning have led to the belief that gradient-based optimization induces a form of implicit regularization, a bias towards models of low ``complexity.''
We study the implicit regularization of gradient descent over deep linear neural networks for matrix completion and sensing, a model referred to as deep matrix factorization.
Our first finding, supported by theory and experiments, is that adding depth to a matrix factorization enhances an implicit tendency towards low-rank solutions, oftentimes leading to more accurate recovery.
Secondly, we present theoretical and empirical arguments questioning a nascent view by which implicit regularization in matrix factorization can be captured using simple mathematical norms.
Our results point to the possibility that the language of standard regularizers may not be rich enough to fully encompass the implicit regularization brought forth by gradient-based optimization.

\end{abstract}

\ifdefined\COLT
	\medskip
\fi

\section{Introduction} \label{sec:intro}

It is a mystery how deep neural networks generalize despite having far more learnable parameters than training examples.
Explicit regularization techniques alone cannot account for this generalization, as they do not prevent the networks from being able to fit random data (see~\cite{zhang2017understanding}). 
A view by which gradient-based optimization induces an \emph{implicit regularization} has thus arisen.
Of course, this view would be uninsightful if ``implicit regularization'' were treated as synonymous with ``promoting generalization''~---~the question is whether we can characterize the implicit regularization independently of any validation data.
Notably, the mere use of the term ``regularization'' already predisposes us towards characterizations based on known explicit regularizers (\eg~a constraint  on some norm of the parameters), but one must also be open to the possibility that something else is afoot.  

An old argument (\cf~\cite{hochreiter1997flat,keskar2017large}) traces implicit regularization in deep learning to beneficial effects of noise introduced by small-batch stochastic optimization. The feeling is that solutions that do not generalize correspond to ``sharp minima,'' and added noise prevents convergence to such solutions. 
However, recent evidence (\eg~\cite{hoffer2017train,you2017scaling}) suggests that deterministic (or near-deterministic) gradient-based algorithms can also generalize, and thus a different explanation is in order.

A major hurdle in this study is that implicit regularization in deep learning seems to kick in only with certain types of data (not with random data for example), and we lack mathematical tools for reasoning about real-life data.
Thus one needs a simple test-bed for the investigation, where data admits a crisp mathematical formulation.
Following earlier works, we focus on the problem of \emph{matrix completion}: given a randomly chosen subset of entries from an unknown matrix~$W^*$, the task is to recover the unseen entries. 
To cast this as a prediction problem, we may view each entry in~$W^*$ as a data point:  observed entries constitute the training set, and the average reconstruction error over the unobserved entries is the test error, quantifying generalization.
Fitting the observed entries is obviously an underdetermined problem with multiple solutions.
However, an extensive body of work (see~\cite{davenport2016overview} for a survey) has shown that if~$W^*$ is low-rank, certain technical assumptions (\eg~``incoherence'') are satisfied and sufficiently many entries are observed, then various algorithms can achieve approximate or even exact recovery.
Of these, a well-known method based upon convex optimization finds the minimal nuclear norm matrix among those fitting all observed entries (see~\cite{candes2009exact}).\note{
Recall that the nuclear norm (also known as trace norm) of a matrix is the sum of its singular values, regarded as a convex relaxation of rank.
}

One may try to solve matrix completion using shallow neural networks. A natural approach, \emph{matrix factorization}, boils down to parameterizing the solution as a product of two matrices~---~$W = W_2 W_1$~---~and optimizing the resulting (non-convex) objective for fitting observed entries.
Formally, this can be viewed as training a depth-$2$ linear neural network. 
It is possible to explicitly constrain the rank of the produced solution by limiting the shared dimension of~$W_1$ and~$W_2$.
However, in practice, even when the rank is unconstrained, running gradient descent with small learning rate (step size) and initialization close to zero tends to produce low-rank solutions, and thus allows accurate recovery if~$W^*$ is low-rank. 
This empirical observation led Gunasekar~\etal~to conjecture in~\cite{gunasekar2017implicit} that gradient descent over a matrix factorization  induces an implicit regularization minimizing nuclear norm:
\begin{conjecture}[from~\cite{gunasekar2017implicit}, informally stated] \label{conj:nuclear}
With small enough learning rate and initialization close enough to the origin, gradient descent on a full-dimensional matrix factorization converges to the minimum nuclear norm solution.
\end{conjecture}

\paragraph{Deep matrix factorization}
Since standard matrix factorization can be viewed as a two-layer neural network (with linear activations), a natural extension is to consider deeper models.
A \emph{deep matrix factorization}\note{
Note that the literature includes various usages of this term~---~some in line with ours (\eg~\cite{trigeorgis2017deep,zhao2017multi,li2015deep}), while others less so (\eg~\cite{xue2017deep,fan2018matrix,wang2017multi}).
}
of~$W \in \R^{d, d'}$, with hidden dimensions $d_1, \ldots, d_{N - 1} \in \N$, is the parameterization:
\be
W = W_N W_{N - 1} \cdots W_1
\text{\,,}
\label{eq:dmf}
\ee
where $W_j \in \R^{d_j, d_{j - 1}}$, $j = 1, \ldots, N$, with $d_N := d, d_0 := d'$.
$N$~is referred to as the \emph{depth} of the factorization, the matrices $W_1, \ldots, W_N$ as its \emph{factors}, and the resulting~$W$ as the \emph{product matrix}. 

Could the implicit regularization of deep matrix factorizations be stronger than that of their shallow counterpart (which Conjecture~\ref{conj:nuclear} equates with nuclear norm minimization)?
Experiments reported in Figure~\ref{fig:exper_intro} suggest that this is indeed the case~---~depth leads to more accurate completion of a low-rank matrix when the number of observed entries is small.
Our purpose in the current paper is to mathematically analyze this stronger form of implicit regularization.
Can it be described by a matrix norm (or quasi-norm) continuing the line of Conjecture~\ref{conj:nuclear}, or is a paradigm shift required? 

\begin{figure}
\begin{center}
\includegraphics[width=\textwidth]{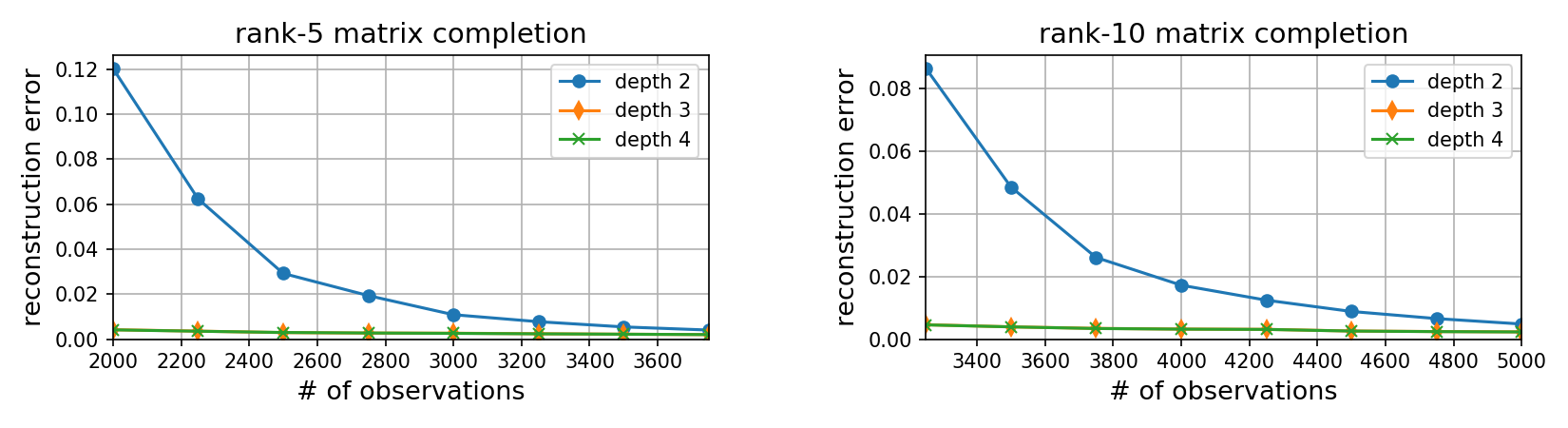}
\end{center}
\caption{
Matrix completion via gradient descent over deep matrix factorizations.
Left (respectively, right) plot shows reconstruction errors for matrix factorizations of depths~$2$, $3$ and~$4$, when applied to the completion of a random rank-$5$ (respectively, rank-$10$) matrix with size $100 \times 100$.
$x$-axis stands for the number of observed entries (randomly chosen), $y$-axis represents reconstruction error, and error bars (indiscernible) mark standard deviations of the results over multiple trials.
All matrix factorizations are full-dimensional, \ie~have hidden dimensions~$100$.
Both learning rate and standard deviation of (random, zero-centered) initialization for gradient descent were set to the small value~$10^{-3}$.
Notice, with few observed entries factorizations of depths~$3$ and~$4$ significantly outperform that of depth~$2$, whereas with more entries all factorizations perform well.
For further details, and a similar experiment on matrix sensing~tasks,~see~Appendix~\ref{app:exper}.
}
\label{fig:exper_intro}
\end{figure}

\subsection{Paper overview} \label{sec:intro:over}

In Section~\ref{sec:norm} we investigate the potential of norms for capturing the implicit regularization in deep matrix factorization. 
Surprisingly, we find that the main theoretical evidence connecting nuclear norm and shallow (depth-$2$) matrix factorization~---~proof given in~\cite{gunasekar2017implicit} for Conjecture~\ref{conj:nuclear} in a particular restricted setting~---~extends to arbitrarily deep factorizations as well.
This result disqualifies the natural hypothesis by which Schatten quasi-norms replace nuclear norm as the implicit regularization when one adds depth to a shallow matrix factorization.
Instead, when interpreted through the lens of~\cite{gunasekar2017implicit}, it brings forth a conjecture by which the implicit regularization is captured by nuclear norm for any depth.
Since our experiments (Figure~\ref{fig:exper_intro}) show that depth changes (enhances) the implicit regularization, we are led to question the theoretical direction proposed in~\cite{gunasekar2017implicit}, and accordingly conduct additional experiments to evaluate the validity of Conjecture~\ref{conj:nuclear}.

Typically, when the number of observed entries is sufficiently large with respect to the rank of the matrix to recover, nuclear norm minimization yields exact recovery, and thus it is impossible to distinguish between that and a different implicit regularization which also perfectly recovers.
The regime most interesting to evaluate is therefore that in which the number of observed entries is too small for exact recovery by nuclear norm minimization~---~here there is room for different implicit regularizations to manifest themselves by providing higher quality solutions.
Our empirical results show that in this regime, matrix factorizations consistently outperform nuclear norm minimization, suggesting that their implicit regularization admits stronger bias towards low-rank, in contrast to Conjecture~\ref{conj:nuclear}.
Together, our theory and experiments lead us to suspect that the implicit regularization in matrix factorization (shallow or deep) may not be amenable to description by a simple mathematical norm, and a detailed analysis of the dynamics in optimization may be necessary.

Section~\ref{sec:dyn} carries out such an analysis, characterizing how the singular value decomposition of the learned solution evolves during gradient descent.
Evolution rates of singular values turn out to be proportional to their size exponentiated by~$2-2/N$, where~$N$ is the depth of the factorization. 
This establishes a tendency towards low rank solutions, which intensifies with depth. 
Experiments validate the findings, demonstrating the dynamic nature of implicit regularization in deep matrix factorization.

We believe the trajectories traversed in optimization may be key to understanding generalization in deep learning, and hope that our work will inspire further progress along this line.

\section{Can the implicit regularization be captured by norms?} \label{sec:norm}

In this section we investigate the possibility of extending Conjecture~\ref{conj:nuclear} for explaining implicit regularization in deep matrix factorization.
Given the experimental evidence in Figure~\ref{fig:exper_intro}, one may hypothesize that gradient descent on a depth-$N$ matrix factorization implicitly minimizes some norm (or quasi-norm) that approximates rank, with the approximation being more accurate the larger $N$ is.
For example, a natural candidate would be Schatten-$p$ quasi-norm to the power of~$p$ ($0 < p \leq 1$), which for a matrix $W \in \R^{d, d'}$ is defined as: $\norm{W}_{S_p}^p := \sum\nolimits_{r = 1}^{\min\{d, d'\}} \sigma_r^p(W)$, where $\sigma_1(W), \ldots, \sigma_{\min\{d, d'\}}(W)$ are the singular values of~$W$.
For~$p = 1$ this reduces to nuclear norm, which by Conjecture~\ref{conj:nuclear} corresponds to a depth-$2$ factorization.
As~$p$ approaches zero we obtain a closer approximation of~$\rank(W)$, which could be suitable for factorizations of higher depths. 

We will focus in this section on \emph{matrix sensing}~---~a more general problem than matrix completion.
Here, we are given~$m$ measurement matrices $A_1, \ldots, A_m$, with corresponding labels $y_1, \ldots, y_m$ generated by $y_i = \inprod{A_i}{W^*}$, and our goal is to reconstruct the unknown matrix~$W^*$.
As in the case of matrix completion, well-known methods, and in particular nuclear norm minimization, can recover~$W^*$ if it is low-rank, certain technical conditions are met, and sufficiently many observations~are~given (see~\cite{recht2010guaranteed}).

\subsection{Current theory does not distinguish depth-$N$ from depth-$2$} \label{sec:norm:theory}

Our first result is that the theory developed by~\cite{gunasekar2017implicit} to support Conjecture~\ref{conj:nuclear} can be generalized to suggest that nuclear norm captures the implicit regularization in matrix factorization not just for depth~$2$, but for arbitrary depth.
This is of course inconsistent with the experimental findings reported in Figure~\ref{fig:exper_intro}.
We will first recall the existing theory, and then show how to extend it. 

\cite{gunasekar2017implicit}~studied implicit regularization in shallow (depth-$2$) matrix factorization by considering recovery of a positive semidefinite matrix from sensing via symmetric measurements, namely:
\be
\min\nolimits_{W \in \S_+^d} \, \ell(W) := \tfrac{1}{2} \sum\nolimits_{i = 1}^m (y_i - \inprod{A_i}{W})^2
\text{\,,}
\label{eq:psd_recover}
\ee
where $A_1, \ldots, A_m \in \R^{d, d}$ are symmetric and linearly independent, and $\S_+^d$ stands for the set of (symmetric and) positive semidefinite matrices in~$\R^{d, d}$.
Focusing on the underdetermined regime~$m \ll d^2$, they investigated the implicit bias brought forth by running \emph{gradient flow} (gradient descent with infinitesimally small learning rate) on a symmetric full-rank matrix factorization, \ie~on the objective:
\[
\psi : \R^{d, d} \to \R_{\geq 0}
\quad,\quad
\psi(Z) := \ell(Z Z^\top) = \tfrac{1}{2} \sum\nolimits_{i = 1}^m (y_i - \inprod{A_i}{Z Z^\top})^2
\text{\,.}
\]
For~$\alpha > 0$, denote by~$W_{\shallow, \infty}(\alpha)$ ($\shallow$~here stands for ``shallow'') the final solution~$Z Z^\top$ obtained from running gradient flow on~$\psi(\cdot)$ with initialization~$\alpha I$ ($\alpha$~times identity).
Formally, $W_{\shallow, \infty}(\alpha) := \lim_{t \to \infty} Z(t) Z(t)^\top$ where $Z(0) = \alpha I$ and $\dot{Z}(t) = -\frac{d \psi}{d Z}(Z(t))$ for~$t \in \R_{\geq 0}$ ($t$~here is a continuous time index, and~$\dot{Z}(t)$ stands for the derivative of~$Z(t)$ with respect to time).
Letting $\norm{\cdot}_*$ represent matrix nuclear norm, the following result was proven by~\cite{gunasekar2017implicit}:
\begin{theorem}[adaptation of Theorem~1 in~\cite{gunasekar2017implicit}] \label{thm:nuclear_mf}
Assume the measurement matrices $A_1, \ldots, A_m$ commute.
Then, if $\bar{W}_\shallow := \lim_{\alpha \to 0} W_{\shallow, \infty}(\alpha)$ exists and is a global optimum for Equation~\eqref{eq:psd_recover} with~$\ell(\bar{W}_\shallow) = 0$, it holds that $\bar{W}_\shallow \in \argmin_{W \in \S_+^d,\ \ell(W) = 0} \norm{W}_*$, \ie~$\bar{W}_\shallow$ is a global optimum with minimal nuclear norm.\note{
The result of~\cite{gunasekar2017implicit} is slightly more general~---~it allows gradient flow to be initialized by~$\alpha O$, where~$O$ is an arbitrary orthogonal matrix, and it is shown that this leads to the exact same~$W_{\shallow, \infty}(\alpha)$ as one would obtain from initializing at~$\alpha I$.
For simplicity, we limit our discussion to the~latter~initialization.
}
\end{theorem}
Motivated by Theorem~\ref{thm:nuclear_mf} and empirical evidence they provided, \cite{gunasekar2017implicit}~raised Conjecture~\ref{conj:nuclear}, which, formally stated, hypothesizes that the condition in Theorem~\ref{thm:nuclear_mf} of~$\{A_i\}_{i = 1}^m$ commuting is unnecessary, and an identical statement holds for arbitrary (symmetric linearly independent) measurement~matrices.\note{
Their conjecture also relaxes the requirement from the initialization of gradient flow~---~an initial value of~$\alpha Z_0$ is believed to suffice, where $Z_0$ is an arbitrary full-rank matrix (that does not depend on~$\alpha$).
}

While the analysis of~\cite{gunasekar2017implicit} covers only symmetric matrix factorizations of the form~$Z Z^\top$, they noted that it can be extended to also account for asymmetric factorizations of the type considered in the current paper.
Specifically, running gradient flow on the objective:
\[
\phi(W_1, W_2) := \ell(W_2 W_1) = \tfrac{1}{2} \sum\nolimits_{i = 1}^m (y_i - \inprod{A_i}{W_2 W_1})^2
\text{\,,}
\]
with $W_1, W_2 \in \R^{d, d}$ initialized to~$\alpha I$, $\alpha > 0$, and denoting by~$W_{\shallow, \infty}(\alpha)$ the product matrix obtained at the end of optimization (\ie~$W_{\shallow, \infty}(\alpha) := \lim_{t \to \infty} W_2(t) W_1(t)$ where $W_j(0) = \alpha I$ and $\dot{W_j}(t) = -\frac{\partial \phi}{\partial W_j}(W_1(t), W_2(t))$ for~$t \in \R_{\geq 0}$), Theorem~\ref{thm:nuclear_mf} holds exactly as stated.
For completeness, we provide a proof of this fact in Appendix~\ref{app:gunasekar_asym}.

\medskip

Next, we show that Theorem~\ref{thm:nuclear_mf}~---~the main theoretical justification for the connection between nuclear norm and shallow matrix factorization~---~extends to arbitrarily deep factorizations as well.
Consider gradient flow over the objective:
\[
\phi(W_1, \ldots, W_N) := \ell(W_N W_{N - 1} \cdots W_1) = \tfrac{1}{2} \sum\nolimits_{i = 1}^m (y_i - \inprod{A_i}{W_N W_{N - 1} \cdots W_1})^2
\text{\,,}
\]
with $W_1, \ldots, W_N \in \R^{d, d}$ initialized to~$\alpha I$, $\alpha > 0$.
Using~$W_{\deep, \infty}(\alpha)$ to denote the product matrix obtained at the end of optimization (\ie~$W_{\deep, \infty}(\alpha) := \lim_{t \to \infty} W_N(t) W_{N - 1}(t) \cdots W_1(t)$ where $W_j(0) = \alpha I$ and $\dot{W_j}(t) = -\frac{\partial \phi}{\partial W_j}(W_1(t), \ldots, W_N(t))$ for~$t \in \R_{\geq 0}$), a result analogous to Theorem~\ref{thm:nuclear_mf} holds:
\begin{theorem} \label{thm:nuclear_dmf}
Suppose $N \geq 3$, and that the matrices $A_1, \ldots, A_m$ commute.
Then, if $\bar{W}_\deep := \lim_{\alpha \to 0} W_{\deep, \infty}(\alpha)$ exists and is a global optimum for Equation~\eqref{eq:psd_recover} with~$\ell(\bar{W}_\deep) = 0$, it holds that $\bar{W}_\deep \in \argmin_{W \in \S_+^d,\,\ell(W) = 0} \norm{W}_*$, \ie~$\bar{W}_\deep$ is a global optimum with minimal nuclear~norm.\note{
By Appendix~\ref{app:proofs:nuclear_dmf}: $W_N(t) W_{N - 1}(t) \cdots W_1(t) \,{\succeq}\, 0 ~~ \forall t$.
Therefore, even though the theorem treats optimization over~$\S_+^d$ using an unconstrained asymmetric factorization, gradient flow implicitly constrains the search to~$\S_+^d$, so the assumption of~$\bar{W}_\deep$ being a global optimum for Equation~\eqref{eq:psd_recover} with~$\ell(\bar{W}_\deep) = 0$ is no stronger than the analogous assumption in Theorem~\ref{thm:nuclear_mf} from~\cite{gunasekar2017implicit}.
The implicit constraining to~$\S_+^d$ also holds when~$N = 2$ (see Appendix~\ref{app:gunasekar_asym}), so the asymmetric extension of Theorem~\ref{thm:nuclear_mf} does not involve strengthening assumptions either.
}
\end{theorem}
\begin{proof}[Proof sketch (for complete proof see Appendix~\ref{app:proofs:nuclear_dmf})] 
Our proof is inspired by that of Theorem~\ref{thm:nuclear_mf} given in~\cite{gunasekar2017implicit}.
Using the expression for~$\dot{W}(t)$ derived in~\citep{arora2018optimization} (Lemma~\ref{lem:end_to_end_dynamics} in Appendix~\ref{app:lemmas}), it can be shown that~$W(t)$ commutes with~$\{A_i\}_{i = 1}^m$, and takes on a particular form.
Taking limits $t \to \infty$ and $\alpha \to 0$, optimality (minimality) of nuclear norm is then established using a duality argument.
\end{proof}

Theorem~\ref{thm:nuclear_dmf} provides a particular setting where the implicit regularization in deep matrix factorizations boils down to nuclear norm minimization.
By Proposition~\ref{prop:schatten_disq} below, there exist instances of this setting for which the minimization of nuclear norm contradicts minimization (even locally) of Schatten-$p$ quasi-norm for any $0 < p < 1$.
Therefore, one cannot hope to capture the implicit regularization in deep matrix factorizations through Schatten quasi-norms.
Instead, if we interpret Theorem~\ref{thm:nuclear_dmf} through the lens of~\cite{gunasekar2017implicit}, we arrive at a conjecture by which the implicit regularization is captured by~nuclear~norm~for~any~depth.

\begin{proposition} \label{prop:schatten_disq}
For any dimension~$d \geq 3$, there exist linearly independent symmetric and commutable measurement matrices $A_1, \ldots, A_m \in \R^{d,d}$, and corresponding labels $y_1, \ldots, y_m \in \R$, such that the limit solution defined in Theorem~\ref{thm:nuclear_dmf}~---~$\bar{W}_\deep$~---~which has been shown to satisfy $\bar{W}_\deep \in \argmin_{W \in \S_+^d, \, \ell(W) = 0} \norm{W}_*$,
is not a local minimum of the following program for any $0 < p < 1$:\note{
Following~\cite{gunasekar2017implicit}, we take for granted existence of~$\bar{W}_\deep$ and it being a global optimum for Equation~\eqref{eq:psd_recover} with~$\ell(\bar{W}_\deep) = 0$.
If this is not the case then Theorem~\ref{thm:nuclear_dmf} does not apply, and hence it obviously does not disqualify minimization of Schatten quasi-norms as the implicit regularization.
}
\[
 \min\nolimits_{W \in \S_+^d, \, \ell(W)=0} \norm{W}_{S_p}
\text{\,.}
\]
\end{proposition}
\begin{proof}[Proof sketch (for complete proof see Appendix~\ref{app:proofs:schatten_disq})]
We choose $A_1, \ldots, A_m$ and $y_1, \ldots, y_m$ such that:
\emph{(i)}~$\bar{W}_\deep = \diag(1, 1, 0, \ldots, 0)$; and
\emph{(ii)}~adding $\epsilon \in (0, 1)$ to entries $(1, 2)$ and~$(2,1)$ of~$\bar{W}_\deep$ maintains optimality.
The result then follows from the fact that the addition of~$\epsilon$ decreases Schatten-$p$ quasi-norm for any $0 < p < 1$.
\end{proof}

\subsection{Experiments challenging Conjecture~\ref{conj:nuclear}} \label{sec:norm:exper} 

Subsection~\ref{sec:norm:theory} suggests that from the perspective of current theory, it is natural to apply Conjecture~\ref{conj:nuclear} to matrix factorizations of arbitrary depth.
On the other hand, the experiment reported in Figure~\ref{fig:exper_intro} implies that depth changes (enhances) the implicit regularization.
To resolve this tension we conduct a more refined experiment, which ultimately puts in question the validity of Conjecture~\ref{conj:nuclear}.

Our experimental protocol is as follows.
For different matrix completion tasks with varying number of observed entries, we compare minimum nuclear norm solution to those brought forth by running gradient descent on matrix factorizations of different depths.
For each depth, we apply gradient descent with different choices of learning rate and standard deviation for (random, zero-centered) initialization, observing the trends as these become smaller.
The outcome of the experiment is presented in Figure~\ref{fig:exper_norm}.
As can be seen, when the number of observed entries is sufficiently large with respect to the rank of the matrix to recover, factorizations of all depths indeed admit solutions that tend to minimum nuclear norm.
However, when there are less entries observed~---~precisely the data-poor setting where implicit regularization matters most~---~neither shallow (depth-$2$) nor deep (depth-$N$ with~$N \geq 3$) factorizations minimize nuclear norm.
Instead, they put more emphasis on lowering the effective rank (\cf~\cite{roy2007effective}), in a manner which is stronger for deeper factorizations.

A close look at the experiments of~\cite{gunasekar2017implicit} reveals that there too, in situations where the number of observed entries (or sensing measurements) was small (less than required for reliable recovery), a discernible gap appeared between the minimal nuclear norm and that returned by (gradient descent on) a matrix factorization.
In light of Figure~\ref{fig:exper_norm}, we believe that if~\cite{gunasekar2017implicit} had included in its plots an accurate surrogate for rank (\eg~effective rank or Schatten-$p$ quasi-norm with small~$p$), scenarios where matrix factorization produced sub-optimal (higher than minimum) nuclear norm would have manifested superior (lower) rank.
More broadly, our experiments suggest that the implicit regularization in (shallow or deep) matrix factorization is somehow geared towards low rank, and just so happens to minimize nuclear norm in cases with sufficiently many observations, where minimum nuclear norm and minimum rank are known to coincide (\cf~\cite{candes2009exact,recht2010guaranteed}).
We note that the theoretical analysis of~\cite{li2018algorithmic} supporting Conjecture~\ref{conj:nuclear} is limited to such cases, and thus cannot truly distinguish between nuclear norm minimization and some other form of implicit regularization favoring low rank.

\begin{figure}
\begin{center}
\includegraphics[width=\textwidth]{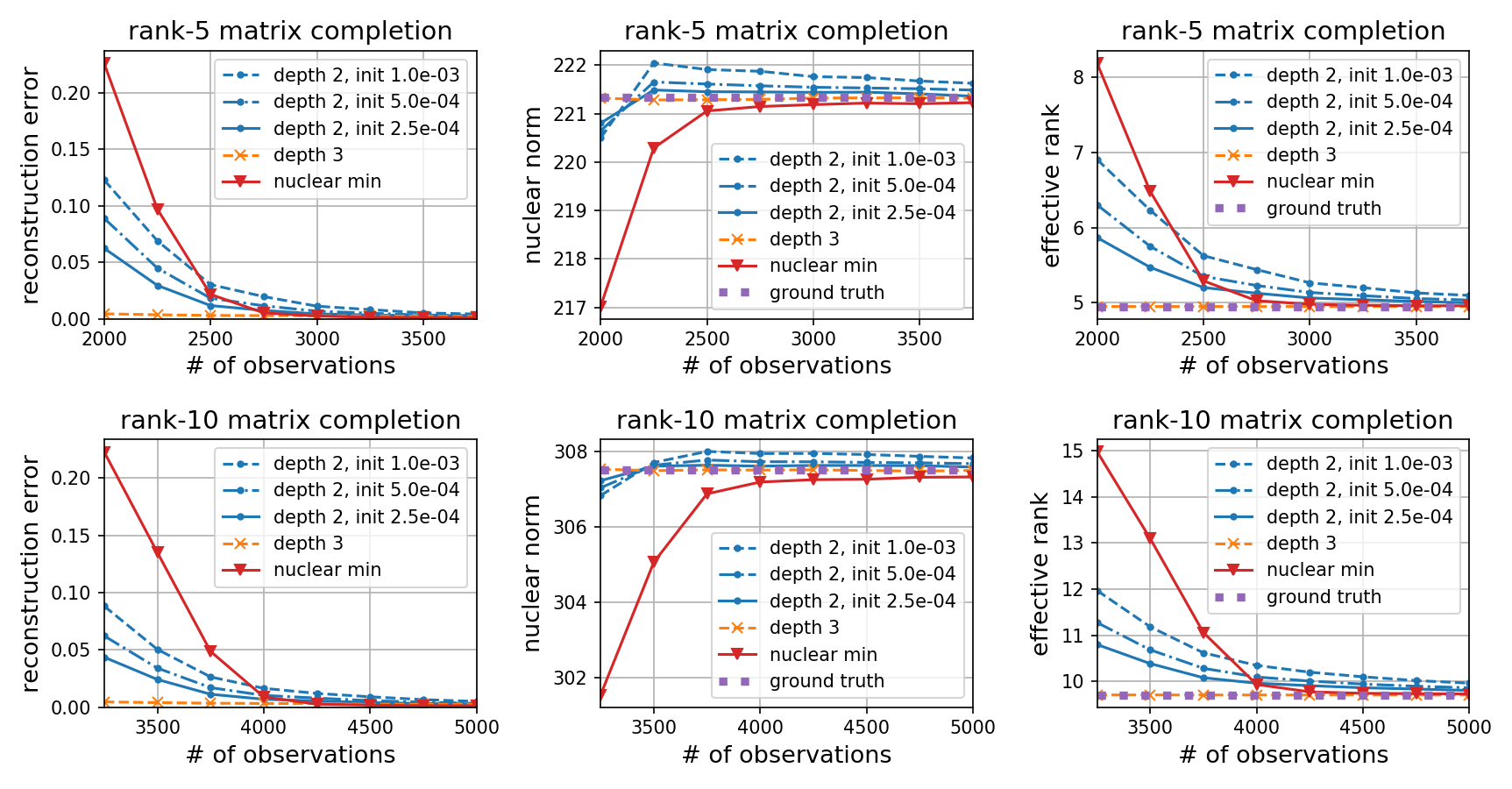}
\end{center}
\caption{
Evaluation of nuclear norm as the implicit regularization in deep matrix factorization.
Each plot compares gradient descent over matrix factorizations of depths $2$ and~$3$ (results for depth~$4$ were indistinguishable from those of depth~$3$; we omit them to reduce clutter) against minimum nuclear norm solution and ground truth in matrix completion tasks. 
Top (respectively, bottom) row corresponds to completion of a random rank-$5$ (respectively, rank-$10$) matrix with size $100 \times 100$.
Left, middle and right columns display (in~$y$-axis) reconstruction error, nuclear norm and effective rank (\cf~\cite{roy2007effective}) respectively.
In each plot, $x$-axis stands for the number of observed entries (randomly chosen), and error bars (indiscernible) mark standard deviations of the results over multiple trials.
All matrix factorizations are full-dimensional, \ie~have hidden dimensions~$100$.
Both learning rate and standard deviation of (random, zero-centered) initialization for gradient descent were initially set to~$10^{-3}$.
Running with smaller learning rate did not yield a noticeable change in terms of final results.
Initializing with smaller standard deviation had no observable effect on results of depth~$3$ (and~$4$), but did impact those of depth~$2$~---~the outcomes of dividing standard deviation by~$2$ and by~$4$ are included in the plots.\protect\footnotemark\,
Notice, with many observed entries minimum nuclear norm solution coincides with ground truth (minimum rank solution), and matrix factorizations of all depths converge to these.
On the other hand, when there are fewer observed entries minimum nuclear norm solution does not coincide with ground truth, and matrix factorizations prefer to lower the effective rank at the expense of higher nuclear norm, in a manner that is more potent for deeper factorizations.
For further details, and a similar experiment on matrix sensing tasks, see Appendix~\ref{app:exper}.
}
\label{fig:exper_norm}
\end{figure}
\notetext{
As can be seen, using smaller initialization enhanced the implicit tendency of depth-$2$ matrix factorization towards low rank.
It is possible that this tendency can eventually match that of depth-$3$ (and~-$4$), but only if initialization size goes far below what is customary in deep learning.
}

Given that Conjecture~\ref{conj:nuclear} seems to hold in some settings (Theorems~\ref{thm:nuclear_mf} and~\ref{thm:nuclear_dmf};~\cite{li2018algorithmic}) but not in other (Figure~\ref{fig:exper_norm}), we hypothesize that capturing implicit regularization in (shallow or deep) matrix factorization through a single mathematical norm (or quasi-norm) may not be possible, and a detailed account for the optimization process might be necessary.
This is carried out in Section~\ref{sec:dyn}.

\section{Dynamical analysis} \label{sec:dyn}

This section characterizes trajectories of gradient flow (gradient descent with infinitesimally small learning rate) on deep matrix factorizations. The characterization significantly extends past analyses for linear neural networks (\eg~\cite{saxe2014exact,arora2018optimization})~---~we derive differential equations governing the dynamics of singular values and singular vectors for the product matrix~$W$ (Equation~\eqref{eq:dmf}). 
Evolution rates of singular values turn out to be proportional to their size exponentiated by~$2-2/N$, where~$N$ is the depth of the factorization. 
For singular vectors, we show that lack of movement implies a particular form of alignment with the gradient, and by this strengthen past results which have only established the converse.
Via theoretical and empirical demonstrations, we explain how our findings imply a tendency towards low-rank solutions, which intensifies with depth.

Our derivation treats a setting which includes matrix completion and sensing as special cases.
We assume minimization of a general analytic loss~$\ell(\cdot)$,\note{
A function~$f(\cdot)$ is \emph{analytic} on a domain~$\D$ if at every $x \in \D$: it is infinitely differentiable; and its Taylor series converges to it on some neighborhood of~$x$ (see~\cite{krantz2002primer} for further details).
} 
overparameterized by a deep matrix factorization:
\vspace{-2mm}
\be
\phi(W_1, \ldots, W_N) := \ell(W_N W_{N-1} \cdots W_1)
\text{\,.}
\label{eq:phi}
\ee
We study gradient flow over the factorization:
\be
\dot{W_j}(t) := \tfrac{d}{dt} W_j(t) = -\tfrac{\partial}{\partial W_j} \phi(W_1(t), \ldots, W_N(t))
\quad,~ t \geq 0 ~,~ j = 1, \ldots, N
\text{\,,}
\label{eq:gf}
\ee
and in accordance with past work, assume that factors are \emph{balanced} at initialization,~\ie:
\be
W_{j + 1}^\top(0) W_{j + 1}(0) = W_j(0) W_j^\top(0)
\quad,~ j = 1, \ldots, N - 1
\text{\,.}
\label{eq:balance}
\ee
Equation~\eqref{eq:balance} is satisfied approximately in the common setting of near-zero initialization (it holds exactly in the ``residual'' setting of identity initialization~---~\cf~\cite{hardt2016identity,bartlett2018gradient}).
The condition played an important role in the analysis of~\cite{arora2018optimization}, facilitating derivation of a differential equation governing the product matrix of a linear neural network (see Lemma~\ref{lem:end_to_end_dynamics} in Appendix~\ref{app:lemmas}).
It was shown in~\cite{arora2018optimization} empirically that there is an excellent match between the theoretical predictions of gradient flow with balanced initialization, and its practical realization via gradient descent with small learning rate and near-zero initialization.
Other works (\eg~\cite{arora2019convergence,ji2019gradient}) later supported this match theoretically.

We note that by Section~6 in~\cite{arora2018optimization}, for depth~$N \geq 3$, the dynamics of the product matrix~$W$ (Equation~\eqref{eq:dmf}) \emph{cannot} be exactly equivalent to gradient descent on the loss~$\ell(\cdot)$ regularized by a penalty term. 
This preliminary observation already hints to the possibility that the effect of depth is different from those of standard regularization techniques.

Employing results of~\cite{arora2018optimization}, we will characterize the evolution of singular values and singular~vectors~for~$W$. 
As a first step, we show that~$W$ admits an \emph{analytic singular value decomposition}~(\cite{bunse1991numerical,de1989analytic}):
\begin{lemma} \label{lemma:asvd}
The product matrix~$W(t)$ can be expressed as:
\be
W(t) = U(t) S(t) V^\top(t)
\text{\,,}
\label{eq:asvd}
\ee
where:
$U(t) \in \R^{d, \min\{d, d'\}}$, $S(t) \in \R^{\min\{d, d'\}, \min\{d, d'\}}$ and $V(t) \in \R^{d', \min\{d, d'\}}$ are analytic functions of~$t$;
and for every~$t$, the matrices $U(t)$~and~$V(t)$ have orthonormal columns, while $S(t)$ is diagonal (elements on its diagonal may be negative and may appear in any order).
\end{lemma}
\begin{proof}[Proof sketch (for complete proof see Appendix~\ref{app:proofs:asvd})]
We show that~$W(t)$ is an analytic function of~$t$ and then invoke Theorem~1 in~\cite{bunse1991numerical}.
\end{proof}
The diagonal elements of~$S(t)$, which we denote by $\sigma_1(t), \ldots, \sigma_{\min\{d, d'\}}(t)$, are signed singular values of~$W(t)$;
the columns of $U(t)$ and~$V(t)$, denoted $\uu_1(t), \ldots, \uu_{\min\{d, d'\}}(t)$ and $\vv_1(t), \ldots, \vv_{\min\{d, d'\}}(t)$, are the corresponding left and right singular vectors (respectively).

\medskip

With Lemma~\ref{lemma:asvd} in place, we are ready to characterize the evolution of singular values:
\begin{theorem} \label{thm:sing_vals_evolve}
The signed singular values of the product matrix~$W(t)$ evolve by:
\be
\dot{\sigma}_r(t) = -N \cdot \big( \sigma_r^2(t) \big)^{1-1/N} \cdot \inprod{\nabla\ell(W(t))}{\uu_r(t) \vv_r^\top(t)}
\quad,~ r = 1, \ldots, \min\{d, d'\}
\text{\,.}
\label{eq:S_evolve}
\ee
If the matrix factorization is non-degenerate, \ie~has depth~$N \geq 2$, the singular values need not be signed (we may assume $\sigma_r(t) \geq 0$ for all~$t$).
\end{theorem}
\begin{proof}[Proof sketch (for complete proof see Appendix~\ref{app:proofs:sing_vals_evolve})]
Differentiating the analytic singular value decomposition (Equation~\eqref{eq:asvd}) with respect to time, multiplying from the left by~$U^\top(t)$ and from the right by~$V(t)$, and using the fact that $U(t)$ and~$V(t)$ have orthonormal columns, we obtain $\dot{\sigma}_r(t) = \uu_r^\top(t) \dot{W}(t) \vv_r(t)$.
Equation~\eqref{eq:S_evolve} then follows from plugging in the expression for~$\dot{W}(t)$ developed by~\cite{arora2018optimization} (Lemma~\ref{lem:end_to_end_dynamics} in Appendix~\ref{app:lemmas}).
\end{proof}
Strikingly, given a value for~$W(t)$, the evolution of singular values depends on~$N$~---~depth of the matrix factorization~---~only through the multiplicative factors~$N \cdot (\sigma_r^2(t))^{1-1/N}$ (see Equation~\eqref{eq:S_evolve}).
In the degenerate case~$N = 1$, \ie~when the product matrix~$W(t)$ is simply driven by gradient flow over the loss~$\ell(\cdot)$ (no matrix factorization), the multiplicative factors reduce to~$1$, and the singular values evolve by: $\dot{\sigma}_r(t) = -\inprod{\nabla\ell(W(t))}{\uu_r(t) \vv_r^\top(t)}$.
With $N \geq 2$, \ie~when depth is added to the factorization, the multiplicative factors become non-trivial, and while the constant~$N$ does not differentiate between singular values, the terms~$(\sigma_r^2(t))^{1-1/N}$ do~---~they enhance movement of large singular values, and on the other hand attenuate that of small ones.
Moreover, the enhancement/attenuation becomes more significant as~$N$ (depth of the factorization) grows.

\medskip

Next, we turn to the evolution of singular vectors:
\begin{lemma} \label{lemma:sing_vecs_evolve}
Assume that at initialization, the singular values of the product matrix~$W(t)$ are distinct and different from zero.\note{
This assumption can be relaxed significantly~---~all that is needed is that no singular value be identically zero ($\forall r \, \exists t ~s.t.~ \sigma_r(t) \neq 0$), and no pair of singular values be identical through time ($\forall r, r' \, \exists t ~s.t.~ \sigma_r(t) \neq \sigma_{r'}(t)$).
}
Then, its singular vectors evolve by:
\bea
\dot{U}(t) &=& - U(t) \left( F(t) \odot \left[ U^\top(t) \nabla\ell(W(t)) V(t) S(t) + S(t) V^\top(t) \nabla\ell^\top(W(t)) U(t) \right] \right) 
\nonumber\\
&& \quad - \left( I_d - U(t) U^\top(t) \right) \nabla\ell(W(t)) V(t) ( S^2(t) )^{\frac{1}{2} - \frac{1}{N}} 
\label{eq:U_evolve}
\eea
\\\vspace{-10mm}
\bea
\dot{V}(t) &=& - V(t) \left( F(t) \odot \left[ S(t) U^\top(t) \nabla\ell(W(t)) V(t) + V^\top(t) \nabla\ell^\top(W(t)) U(t) S(t) \right] \right) 
\nonumber\\
&& \quad - \left( I_{d'} - V(t) V^\top(t) \right) \nabla\ell^\top(W(t)) U^\top(t) ( S^2(t) )^{\frac{1}{2} - \frac{1}{N}}
\label{eq:V_evolve}
\text{\,,}
\eea
where $I_d$ and~$I_{d'}$ are the identity matrices of sizes $d \times d$ and~$d' \times d'$ respectively, $\odot$~stands for the Hadamard (element-wise) product, and the matrix~$F(t) \in \R^{\min\{d, d'\}, \min\{d, d'\}}$ is skew-symmetric with $( ( \sigma_{r'}^2(t) )^{1 / N} - ( \sigma_r^2(t) )^{1 / N} )^{- 1}$ in its $(r, r')$'th entry, $r \neq r'$.\note{
Equations~\eqref{eq:U_evolve} and~\eqref{eq:V_evolve} are well-defined when~$t$ is such that $\sigma_1(t), \ldots, \sigma_{\min\{d, d'\}}(t)$ are distinct and different from zero.
By analyticity, this is either the case for every~$t$ besides a set of isolated points, or it is not the case for any~$t$.
Our assumption on initialization disqualifies the latter option, so any~$t$ for which Equations~\eqref{eq:U_evolve} or~\eqref{eq:V_evolve} are ill-defined is isolated.
The derivatives of $U$ and~$V$ for such~$t$ may thus be inferred by continuity.
}
\end{lemma}
\begin{proof}[Proof sketch (for complete proof see Appendix~\ref{app:proofs:sing_vecs_evolve})]
We follow a series of steps adopted from~\cite{townsend2016differentiating} to obtain expressions for $\dot{U}(t)$ and~$\dot{V}(t)$ in terms of $U(t)$, $V(t)$, $S(t)$ and~$\dot{W}(t)$.
Plugging in the expression for~$\dot{W}(t)$ developed by~\cite{arora2018optimization} (Lemma~\ref{lem:end_to_end_dynamics} in Appendix~\ref{app:lemmas}) then yields Equations~\eqref{eq:U_evolve},~\eqref{eq:V_evolve}.
\end{proof}
\begin{corollary} \label{cor:sing_vecs_station}
Assume the conditions of Lemma~\ref{lemma:sing_vecs_evolve}, and that the matrix factorization is non-degenerate, \ie~has depth~$N \geq 2$.
Then, for any time~$t$ such that the singular vectors of the product matrix~$W(t)$ are stationary, \ie~$\dot{U}(t) = 0$ and $\dot{V}(t) = 0$, it holds that $U^\top(t) \nabla\ell(W(t)) V(t)$ is diagonal, meaning they align with the singular vectors of $\nabla\ell(W(t))$.
\end{corollary}
\begin{proof}[Proof sketch (for complete proof see Appendix~\ref{app:proofs:sing_vecs_station})]
By Equations~\eqref{eq:U_evolve} and~\eqref{eq:V_evolve}, $U^\top(t) \dot{U}(t) S(t) - S(t) V^\top(t) \dot{V}(t)$ is equal to the Hadamard product between $U^\top(t) \nabla\ell(W(t)) V(t)$ and a (time-dependent) square matrix with zeros on its diagonal and non-zeros elsewhere.
When $\dot{U}(t) = 0$ and $\dot{V}(t) = 0$ obviously $U^\top(t) \dot{U}(t) S(t) - S(t) V^\top(t) \dot{V}(t) = 0$, and so the Hadamard product is zero.
This implies that $U^\top(t) \nabla\ell(W(t)) V(t)$ is diagonal.
\end{proof}
Earlier papers studying gradient flow for linear neural networks (\eg~\cite{saxe2014exact,advani2017high,lampinen2019analytic}) could show that singular vectors are stationary if they align with the singular vectors of the gradient.
Corollary~\ref{cor:sing_vecs_station} is significantly stronger and implies a converse~---~if singular vectors are stationary, they must be aligned with the gradient.
Qualitatively, this suggests that a ``goal'' of gradient flow on a deep matrix factorization is to align singular vectors of the product matrix with those of the gradient. 

\subsection{Implicit regularization towards low rank}

Figure~\ref{fig:exper_dyn} presents empirical demonstrations of our conclusions from Theorem~\ref{thm:sing_vals_evolve} and Corollary~\ref{cor:sing_vecs_station}.
It shows that for a non-degenerate deep matrix factorization, \ie~one with depth~$N \geq 2$, under gradient descent with small learning rate and near-zero initialization, singular values of the product matrix are subject to an enhancement/attenuation effect as described above: 
they progress very slowly after initialization, when close to zero;
then, upon reaching a certain threshold, the movement of a singular value becomes rapid, with the transition from slow to rapid movement being sharper with a deeper factorization (larger~$N$).
In terms of singular vectors, the figure shows that those of the product matrix indeed align with those of the gradient.
Overall, the dynamics promote solutions that have a few large singular values and many small ones, with a gap that is more extreme the deeper the matrix factorization is.
This is an implicit regularization towards low rank, which intensifies with depth.

\begin{figure}
\begin{center}
\includegraphics[width=\textwidth]{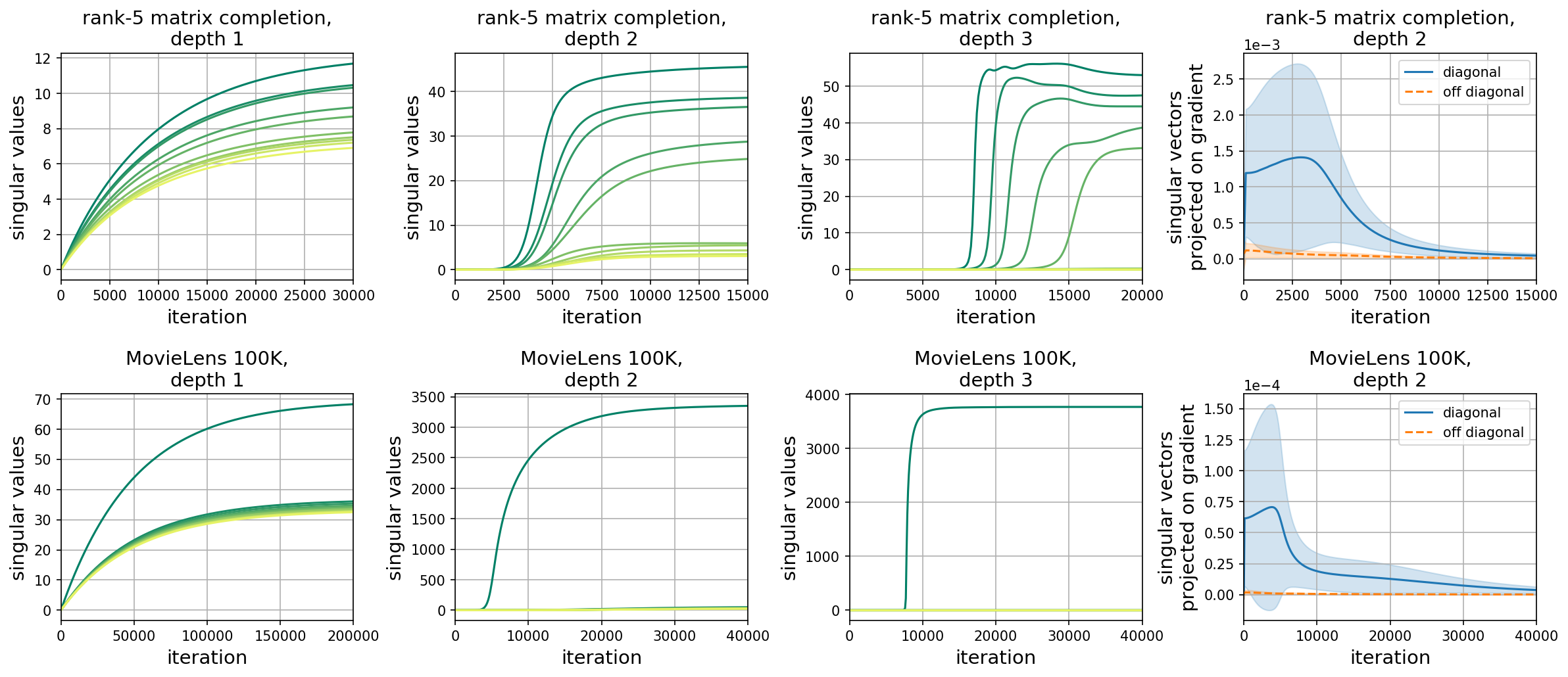}
\end{center}
\caption{
Dynamics of gradient descent over deep matrix factorizations~---~specifically, evolution of singular values and singular vectors of the product matrix during training for matrix completion.
Top row corresponds to the task of completing a random rank-$5$ matrix with size $100 \times 100$ based on $2000$ randomly chosen observed entries; bottom row corresponds to training on $10000$ entries chosen randomly from the MovieLens~100K dataset (completion of a~$943 \times 1682$~matrix, \cf~\cite{harper2016movielens}).\protect\footnotemark\,
First (left) three columns show top singular values for, respectively, depths~$1$~(no matrix factorization), $2$~(shallow matrix factorization) and~$3$~(deep matrix factorization).
Last (right) column shows singular vectors for a depth-$2$ factorization, by comparing on-~\vs~off-diagonal entries in the matrix $U^\top(t) \nabla\ell(W(t)) V(t)$ (see Corollary~\ref{cor:sing_vecs_station})~---~for each group of entries, mean of absolute values is plotted, along with shaded area marking the standard deviation.
All matrix factorizations are full-dimensional (hidden dimensions $100$ in top row plots, $943$~in bottom row plots).
Notice, increasing depth makes singular values move slower when small and faster when large (in accordance with Theorem~\ref{thm:sing_vals_evolve}), which results in solutions with effectively lower rank.
Notice also that $U^\top(t) \nabla\ell(W(t)) V(t)$ is diagonally dominant so long as there is movement, showing that singular vectors of the product matrix align with those of the gradient (in accordance with Corollary~\ref{cor:sing_vecs_station}).
For further details, and a similar experiment on matrix sensing, see Appendix~\ref{app:exper}.
}
\label{fig:exper_dyn}
\end{figure}
\notetext{
Observations of MovieLens~100K were subsampled solely for reducing run-time.
}

\paragraph{Theoretical illustration}
Consider the simple case of square matrix sensing with a single measurement fit via $\ell_2$~loss: $\ell(W) = \frac{1}{2} (\inprod{A}{W} - y)^2$, where $A \in \R^{d, d}$ is the measurement matrix, and $y \in \R$ the corresponding label.
Suppose we learn by running gradient flow over a depth-$N$ matrix factorization, \ie~over the objective~$\phi(\cdot)$ defined in Equation~\eqref{eq:phi}.
Corollary~\ref{cor:sing_vecs_station} states that the singular vectors of the product matrix~---~$\{\uu_r(t)\}_r$ and~$\{\vv_r(t)\}_r$~---~are stationary only when they diagonalize the gradient, meaning $\left\{ | \uu_r^\top(t) \nabla\ell(W(t)) \vv_r | : r = 1, \ldots, d \right\}$ coincides with the set of singular values in~$\nabla\ell(W(t))$.
In our case $\nabla\ell(W) = (\inprod{A}{W} - y) A$, so stationarity of singular vectors implies $| \uu_r^\top(t) \nabla\ell(W(t)) \vv_r | = | \delta(t) | \cdot \rho_r$, where $\delta(t) := \inprod{A}{W(t)} - y$ and $\rho_1, \ldots, \rho_d$ are the singular values of~$A$ (in no particular order).
We will assume that starting from some time~$t_0$ singular vectors are stationary, and accordingly $\uu_r^\top(t) \nabla\ell(W(t)) \vv_r(t) = \delta(t) \cdot e_r \cdot \rho_r$ for $r = 1, \ldots, d$, where $e_1, \ldots, e_d \in \{-1, 1\}$.
Theorem~\ref{thm:sing_vals_evolve} then implies that (signed) singular values of~the~product~matrix~evolve~by: 
\be
\dot{\sigma}_r(t) = -N \cdot \big( \sigma_r^2(t) \big)^{1-1/N} \cdot \delta(t) \cdot e_r \cdot \rho_r
\quad,~ \forall t \geq t_0
\text{\,.}
\label{eq:sing_vals_evolve_examp}
\ee
Let $r_1, r_2 \in \{1, \ldots, d\}$.
By Equation~\eqref{eq:sing_vals_evolve_examp}:
\[
\int_{t' = t_0}^{t} \big( \sigma_{r_1}^2(t') \big)^{- 1 + 1/N} \dot{\sigma}_{r_1}(t') dt' = \frac{e_{r_1} \rho_{r_1}}{e_{r_2} \rho_{r_2}} \cdot \int_{t' = t_0}^{t} \big( \sigma_{r_2}^2(t') \big)^{- 1 + 1/N} \dot{\sigma}_{r_2}(t') dt'
\text{\,.}
\]
Computing the integrals, we may express~$\sigma_{r_1}(t)$ as a function of~$\sigma_{r_2}(t)$:\note{
In accordance with Theorem~\ref{thm:sing_vals_evolve}, if $N \geq 2$, we assume without loss of generality that $\sigma_{r_1}(t), \sigma_{r_2}(t) \geq 0$, while disregarding the trivial case of equality to zero.
}
\be
\sigma_{r_1}(t) =
\begin{cases}
\alpha_{r_1, r_2} \cdot \sigma_{r_2}(t) + const & , N = 1 \\[1mm]
\big( \sigma_{r_2}(t) \big)^{\alpha_{r_1, r_2}} \cdot const & , N = 2 \\[-0.5mm]
\big( \alpha_{r_1, r_2} \cdot ( \sigma_{r_2}(t) )^{- \frac{N - 2}{N}} + const \big)^{-\frac{N}{N - 2}} & , N \geq 3
\end{cases}
\text{\,,}
\label{eq:sing_vals_depend_examp}
\ee
where~$\alpha_{r_1, r_2} := e_{r_1} \rho_{r_1} ( e_{r_2} \rho_{r_2} )^{-1}$, and $const$ stands for a value that does not depend on~$t$.
Equation~\ref{eq:sing_vals_depend_examp} reveals a gap between~$\sigma_{r_1}(t)$ and~$\sigma_{r_2}(t)$ that enhances with depth.
For example, consider the case where $0 < \alpha_{r_1, r_2} < 1$.
If the depth~$N$ is one, \ie~the matrix factorization is degenerate, $\sigma_{r_1}(t)$~will grow linearly with~$\sigma_{r_2}(t)$.
If~$N = 2$~---~shallow matrix factorization~---~$\sigma_{r_1}(t)$~will grow polynomially more slowly than~$\sigma_{r_2}(t)$ ($const$~here is positive).
Increasing depth further will lead $\sigma_{r_1}(t)$ to asymptote when $\sigma_{r_2}(t)$ grows, at a value which can be shown to be lower the larger~$N$~is.
Overall, adding depth to the matrix factorization leads to more significant gaps between singular values of the product matrix, \ie~to a stronger implicit bias towards low rank.

\section{Related work} \label{sec:related}

Implicit regularization in deep learning is a highly active area of research.
For non-linear neural networks, the topic has thus far been studied empirically (\eg~in~\cite{neyshabur2014search,zhang2017understanding,keskar2017large,hoffer2017train,neyshabur2017exploring}), with theoretical analyses being somewhat scarce (see~\cite{du2018algorithmic,rahaman2018spectral} for some of the few observations that have been derived).
The majority of theoretical attention has been devoted to (single-layer) linear predictors and (multi-layer) linear neural networks, often viewed as stepping stones towards non-linear models.
Linear predictors were treated in~\cite{lin2016generalization,soudry2018implicit,nacson2019convergence,gunasekar2018characterizing}.
For linear neural networks, \cite{advani2017high,lampinen2019analytic,gidel2019implicit}~studied settings where the training objective admits a single global minimum, and the question is what path gradient descent (or gradient flow) takes to reach it.\note{
\cite{advani2017high}~and~\cite{gidel2019implicit} also considered settings where there are multiple global minima, but in these too there was just one solution to which optimization could converge, leaving only the question of what path is taken to reach it.
}
This stands in contrast to the practical deep learning scenario where there are multiple global minima, and implicit regularization refers to the optimizer being biased towards reaching those solutions that generalize well.
The latter scenario was treated by \cite{gunasekar2018implicit} and~\cite{ji2019gradient} in the context of linear neural networks trained for binary classification via separable data.
These works showed that under certain assumptions, gradient descent converges (in direction) to the maximum margin solution.
Intriguingly, the bias towards maximum margin holds with any number of layers, so in particular, implicit regularization was found to be oblivious~to~depth.\note{
In addition to standard linear neural networks, \cite{gunasekar2018implicit}~also analyzed ``linear convolutional networks'', characterized by a particular weight sharing pattern.
For such models, the implicit regularization was found to promote sparsity in the frequency domain, in a manner which does depend on depth.
}

The most extensively studied instance of linear neural networks is matrix factorization, corresponding to a model with multiple inputs, multiple outputs and a single hidden layer, typically trained to recover a low-rank linear mapping.
The literature on matrix factorization for low-rank matrix recovery is far too broad to cover here~---~we refer to~\cite{chi2018nonconvex} for a recent survey, while mentioning that the technique is oftentimes attributed to~\cite{burer2003nonlinear}.
Notable works proving successful recovery of a low-rank matrix through matrix factorization trained by gradient descent with no explicit regularization are~\cite{tu2016low,ma2018implicit,li2018algorithmic}.
Of these, \cite{li2018algorithmic}~can be viewed as resolving the conjecture of~\cite{gunasekar2017implicit}~---~which we investigate in Section~\ref{sec:norm}~---~for the case of sufficiently many linear measurements satisfying the restricted isometry property.

To the best of our knowledge, the current paper is the first to study implicit regularization for deep (three or more layer) linear neural networks with multiple outputs.
The latter trait seems to be distinctive, as it is the main differentiator between the setting of~\cite{gunasekar2018implicit,ji2019gradient}, where implicit regularization is oblivious to depth, and ours, for which we show that depth has significant impact.
We note that our work is focused on the type of solutions reached by gradient descent, not the complementary questions of whether an optimal solution is found, and how fast that happens.
These questions were studied extensively for matrix factorization~---~\cf~\cite{ge2016matrix,bhojanapalli2016global,park2017nonsquare,ge2017no}~---~and more recently for linear neural networks of arbitrary depth~---~see~\cite{bartlett2018gradient,arora2018optimization,arora2019convergence,du2019width}.
From a technical perspective, closest to our work are~\cite{gunasekar2017implicit} and~\cite{arora2018optimization}~---~we rely on their results and significantly extend them (see Sections~\ref{sec:norm} and~\ref{sec:dyn}).

\section{Conclusion} \label{sec:conclu}

The implicit regularization of gradient-based optimization is key to generalization in deep learning.
As a stepping stone towards understanding this phenomenon, we studied deep linear neural networks for matrix completion and sensing, a model referred to as deep matrix factorization.
Through theory and experiments, we questioned prevalent norm-based explanations for implicit regularization in matrix factorization (\cf~\cite{gunasekar2017implicit}), and offered an alternative, dynamical approach.
Our characterization of the dynamics induced by gradient flow on the singular value decomposition of the learned matrix significantly extends prior work on linear neural networks.
It reveals an implicit tendency towards low rank which intensifies with depth, supporting the empirical superiority of deeper matrix factorizations.

An emerging view is that understanding optimization in deep learning necessitates a detailed account for the trajectories traversed in training (\cf~\cite{arora2019convergence}). 
Our work adds another dimension to the potential importance of trajectories~---~we believe they are necessary for understanding generalization as well, and in particular, may be key to analyzing implicit regularization for non-linear neural networks.

\newcommand{\ack}
{This work was supported by NSF, ONR, Simons Foundation, Schmidt Foundation, Mozilla Research, Amazon Research, DARPA and SRC.
Nadav Cohen was a member at the Institute for Advanced Study, and was additionally supported by the Zuckerman Israeli Postdoctoral Scholars Program.
The authors thank Nathan Srebro for illuminating discussions which helped improve the paper.}
\ifdefined\COLT
	\acks{\ack}
\else
	\ifdefined\CAMREADY
		\ifdefined\ICLR
			\newcommand*{\subsuback}{}
		\fi
		\ifdefined\NEURIPS
			\newcommand*{\subsuback}{}
		\fi
		\ifdefined\subsuback		
			\subsubsection*{Acknowledgments}
		\else
			\section*{Acknowledgments}
		\fi
		\ack
	\fi
\fi

\section*{References}
{\small
\ifdefined\ICML
	\bibliographystyle{icml2018}
\else
	\bibliographystyle{plainnat}
\fi
\bibliography{refs.bib}
}

\clearpage
\appendix

\ifdefined\NOTESAPP
	\theendnotes
\fi

\section{Useful lemmas} \label{app:lemmas}

We recall the following result from~\cite{arora2018optimization}, which characterizes the evolution of the product matrix under gradient flow on a deep matrix factorization:
\begin{lemma}[adaptation of Theorem~1 in~\cite{arora2018optimization}] \label{lem:end_to_end_dynamics}
Let $\ell : \R^{d, d'} \to \R_{\geq 0}$ be a continuously differentiable loss, overparameterized by a deep matrix factorization:
\[
\phi(W_1, \ldots, W_N) = \ell(W_N W_{N-1} \cdots W_1)
\text{\,.}
\]
Suppose we run gradient flow over the factorization:
\[
\dot{W_j}(t) := \tfrac{d}{dt} W_j(t) = -\tfrac{\partial}{\partial W_j} \phi(W_1(t), \ldots, W_N(t))
\quad,~ t \geq 0 ~,~ j = 1, \ldots, N
\text{\,,}
\]
with factors initialized to be balanced,~\ie:
\[
W_{j + 1}^\top(0) W_{j + 1}(0) = W_j(0) W_j^\top(0)
\quad,~ j = 1, \ldots, N - 1
\text{\,.}
\]
Then, the product matrix $W(t) = W_N(t) \cdots W_1(t)$ obeys the following dynamics:
\[
\dot{W}(t) = -\sum\nolimits_{j = 1}^N \left[ W(t) W^\top(t) \right]^\frac{j - 1}{N} \cdot \nabla \ell\big(W(t)\big) \cdot \left[ W^\top(t) W(t) \right]^\frac{N - j}{N}
\text{\,,}
\]
where~$[\,\cdot\,]^\alpha$, $\alpha \in \R_{\geq 0}$, stands for a power operator defined over positive semidefinite matrices (with $\alpha = 0$ yielding identity by definition).
\end{lemma}

\medskip

An additional result we will use is the following technical lemma:
\begin{lemma} \label{lemma:sign_preserve}
Let $\alpha \,{\geq}\, \frac{1}{2}$ and $g \,{:}\, [0, \infty) {\to}\, \R$ be a continuous function.
Consider the initial value problem:
\vspace{-1.5mm}
\be
s(0) = s_0
~~~,~~~ \dot{s}(t) = (s^2(t))^\alpha \cdot g(t) 
~~~ \forall t \geq 0
\text{\,,}
\label{eq:sign_preserve}
\ee
\vspace{-3mm}\\
where~$s_0 \in \R$.
Then, as long as it does not diverge to~$\pm\infty$, the solution to this problem~($s(t)$) has the same sign as its initial value~($s_0$).
That is, $s(t)$~is identically zero if $s_0 = 0$, is positive if $s_0 > 0$, and is negative if $s_0 < 0$.
\end{lemma}
\begin{proof}
If $\alpha = 1 / 2$, the solution to Equation~\eqref{eq:sign_preserve} is:
\[
s(t) =
\begin{cases}
~~ s_0 \cdot \exp \big( \int_{t' = 0}^t g(t') dt' \big) & , s_0 > 0 \\[1.5mm]
s_0 \cdot \exp \big( - \int_{t' = 0}^t g(t') dt' \big) & , s_0 < 0 \\[2mm]
~\quad\qquad\qquad 0 & , s_0 = 0
\end{cases}
\text{\,.}
\]
This solution does not diverge in finite time (regardless of the chosen~$g(\cdot)$), and obviously preserves the sign of its initial value.

If~$\alpha > 1/2$, Equation~\eqref{eq:sign_preserve} is solved by:
\[
s(t) =
\begin{cases}
~\quad \big( s_0^{-2\alpha + 1} + (-2\alpha + 1) \int_{t' = 0}^t g(t') dt' \big)^{\frac{1}{-2\alpha + 1}} & , s_0 > 0 \\[1.5mm]
- \big( (- s_0)^{-2\alpha + 1} - (-2\alpha + 1) \int_{t' = 0}^t g(t') dt' \big)^{\frac{1}{-2\alpha + 1}} & , s_0 < 0 \\[2mm]
~\qquad\qquad\qquad\qquad\qquad 0 & , s_0 = 0
\end{cases}
\text{\,.}
\]
In this case, divergence in finite time can take place (depending on the choice of~$g(\cdot)$), but nonetheless the sign of~$s(t)$ is preserved until that happens.
\end{proof}

\section{Deferred proofs} \label{app:proofs}

\subsection{Proof of Theorem~\ref{thm:nuclear_dmf}} \label{app:proofs:nuclear_dmf}

For convenience, throughout the proof we replace the notation~$\bar{W}_\deep$ by~$W^*_\deep$.
We also define a linear operator~$\A$ which specifies all~$m$ measurements:
\[
\A:\R^{d,d}\to\R^m \quad, \quad
\A(W) =
\begin{pmatrix}
\langle A_1, W \rangle\\
\vdots\\
\langle A_m, W \rangle
\end{pmatrix} \,,
\]
and its adjoint operator~$\A^\dagger$:
\[
\A^\dagger:\R^{m}\to\R^{d,d} \quad, \quad
\A^\dagger(\rr) = \sum_{i=1}^m r_i A_i \,.
\]
Then we can rewrite the loss function in Equation~\eqref{eq:psd_recover} as:
\[
\ell(W) = \frac12 \norm{\A(W)-\y}_2^2 \,,
\]
where $\y := (y_1, \ldots, y_m)^\top \in \R^m$.
The gradient of $\ell(\cdot)$ can be expressed as:
\[
\nabla\ell(W) = \A^\dagger(\A(W) - \y).
\]

We consider a fixed $\alpha>0$ for now, and will take the limit $\alpha\to0^+$ later.
Recall that gradient flow is run on the objective $\phi(W_1, \ldots, W_N) = \ell(W_N\cdots W_1)$, with initialization $W_j(0) = \alpha I,~j=1,\ldots,N$.
From Lemma~\ref{lem:end_to_end_dynamics}, we know that the product matrix $W(t)=W_N(t)\cdots W_1(t)$ evolves by:
\be \label{eqn:e2e_dynamics_in_norm_proof}
\begin{aligned}
    \dot{W}(t) &= -\sum\nolimits_{j = 1}^N \left[ W(t) W^\top(t) \right]^\frac{j - 1}{N} \cdot \nabla \ell\big(W(t)\big) \cdot \left[ W^\top(t) W(t) \right]^\frac{N - j}{N} \\
    &= -\sum\nolimits_{j = 1}^N \left[ W(t) W^\top(t) \right]^\frac{j - 1}{N} \cdot \A^\dagger(\rr(t)) \cdot \left[ W^\top(t) W(t) \right]^\frac{N - j}{N} \quad,~ t \in \R_{\geq 0} \,, \\[2mm]
    W(0) &= \alpha^N I \,,
\end{aligned}
\ee
where $\rr(t) := \A(W(t))-\y$ is the vector of residuals at time $t$.
Since $A_1, \ldots, A_m$ are symmetric and commutable, they are simultaneously (orthogonally) diagonalizable, \ie~there exists an orthogonal matrix $O\in\R^{d,d}$ such that $\tilde{A_i} := O A_i O^\top$, $i = 1, \ldots, m$, are all diagonal.
Consider a change of variables $\tilde{W}(t) := OW(t)O^\top$, and denote $\tilde{\A}^\dagger(\rr) := O\A^\dagger(\rr)O^\top = \sum_{i=1}^mr_i\tilde{A_i}$.
Then it follows from Equation~\eqref{eqn:e2e_dynamics_in_norm_proof} that:
\be \label{eqn:e2e_diagonalized_dynamics_in_norm_proof}
\begin{aligned}
    \dot{\tilde W}(t)
    &= -\sum\nolimits_{j = 1}^N \big[ \tilde{W}(t) \tilde{W}^\top(t) \big]^\frac{j - 1}{N} \cdot \tilde{\A}^\dagger(\rr(t)) \cdot \big[ \tilde{W}^\top(t) \tilde{W}(t) \big]^\frac{N - j}{N} \quad,~ t \in \R_{\geq 0} \,, \\[2mm]
    \tilde{W}(0) &= \alpha^N I \,.
\end{aligned}
\ee
Notice that: 
\emph{(i)}~$\tilde{W}(0)$ is diagonal;
and \emph{(ii)}~if~$\tilde{W}(t)$ is diagonal then so is~$\dot{\tilde W}(t)$.
We may therefore set the off-diagonal elements of~$\tilde{W}(t)$ to zero, and solve for the diagonal ones:
\be \label{eqn:e2e_diagonals_dynamics_in_norm_proof}
\begin{aligned}
    \dot{\tilde W}_{kk}(t) = -N \cdot \left(\tilde{W}_{kk}^2(t)\right)^{\frac{N-1}{N}} \cdot \tilde{\A}^\dagger_{kk}(\rr(t)) \,,
    ~~ \tilde{W}_{kk}(0) = \alpha^N \,,
    ~~ t \in \R_{\geq 0} \,,
    ~~ k=1, \ldots, d \,.
\end{aligned}
\ee
By Lemma~\ref{lemma:sign_preserve}, $\tilde{W}_{kk}(t)$~maintains the sign of its initialization, meaning it stays positive.
Moreover, since by assumption $N \geq 3$, the solution to Equation~\eqref{eqn:e2e_diagonals_dynamics_in_norm_proof} is:
\begin{align*}
    \tilde{W}_{kk}(t) &= \alpha^N \left( 1 + (N-2)\alpha^{N-2}\cdot\tilde{\A}^\dagger_{kk}\left(\s(t)\right) \right)^{-\frac{N}{N-2}} \,, 
    \quad t \in \R_{\geq 0} \,,    
    \quad k=1,\ldots,d \,,
\end{align*}
where $\s(t) := \int_{t' = 0}^t (\rr(t')) dt'$.
The matrix~$\tilde{W}(t)$ thus has positive elements on its diagonal (and zeros elsewhere), and takes the following form:
\be
\tilde{W}(t) = \alpha^N \left[ I_d + (N-2)\alpha^{N-2}\cdot\tilde{\A}^\dagger\left(\s(t)\right) \right]^{-\frac{N}{N-2}} 
\text{\,,}
\label{eqn:e2e_diagonalized_solution_in_norm_proof}
\ee
where~$I_d$ is the $d \times d$ identity matrix, and $[\,\cdot\,]^{- N / (N - 2)}$ is a negative power operator defined over positive definite matrices.
We assume $W_{\deep, \infty}(\alpha) := \lim_{t \to \infty} W(t)$ exists, and so we may write:
\[
\tilde{W}_{\deep, \infty}(\alpha) := O \, W_{\deep, \infty}(\alpha) \, O^\top = \lim_{t \to \infty} \tilde{W}(t) = \alpha^N \left[ I_d - \tilde{\A}^\dagger\left(\bnu_\infty(\alpha)\right) \right]^{-\frac{N}{N-2}} 
\text{\,,}
\]
where $\bnu_\infty(\alpha) := - (N-2) \alpha^{N-2} \lim_{t \to \infty} \s(t)$.\note{
Existence of $\lim_{t \to \infty} \tilde{W}(t)$ implies that $\lim_{t \to \infty} \tilde{\A}^\dagger\left(\s(t)\right)$ exists (see Equation~\eqref{eqn:e2e_diagonalized_solution_in_norm_proof}), which in turn, given the linear independence of~$\{ \tilde{A}_i \}_{i=1}^m$, indicates that $\lim_{t \to \infty} \s(t)$ exists as well.
Similarly to~\cite{gunasekar2017implicit}, the remainder of the proof treats the case where $\lim_{t \to \infty} \s(t)$ is finite, thereby avoiding the technical load associated with infinite coordinates.
}
Since $\{\tilde{W}(t)\}_t$ are diagonal, $\tilde{W}_{\deep, \infty}(\alpha)$~is diagonal.
Additionally, positive definiteness of~$\{\tilde{W}(t)\}_t$ implies that $\tilde{W}_{\deep, \infty}(\alpha)$~is positive definite as well (it cannot have zero eigenvalues as it is given by a negative power operator).
This means that:
\be
I_d - \tilde{\A}^\dagger\left(\bnu_\infty(\alpha)\right) = \alpha^{- N} \big[ \tilde{W}_{\deep, \infty}(\alpha) \big]^{-\frac{N - 2}{N}} \succ 0
\implies
\tilde{\A}^\dagger\left(\bnu_\infty(\alpha)\right) \prec I_d
\label{eq:nu_admissible}
\text{\,.}
\ee
Now take the limit $\alpha \to 0^+$.
By assumption $W^*_\deep := \lim_{\alpha \to 0^+} W_{\deep, \infty}(\alpha)$ exists, so we can write:
\[
\tilde{W}^*_\deep := O \, W^*_\deep \, O^\top = \lim_{\alpha \to 0^+} \tilde{W}_{\deep, \infty}(\alpha) = \lim_{\alpha \to 0^+} \alpha^N \left[ I_d - \tilde{\A}^\dagger\left(\bnu_\infty(\alpha)\right) \right]^{-\frac{N}{N-2}}
\text{\,,}
\]
The fact that $\{\tilde{W}_{\deep, \infty}(\alpha)\}_\alpha$ are diagonal and positive definite implies that:
\be
\tilde{W}^*_\deep \succeq 0
\label{eq:W_final_diag_psd}
\text{\,.}
\ee
Moreover, if the $k$'th element on the diagonal of~$\tilde{W}^*_\deep$ is non-zero, it must hold that:
\[
\lim_{\alpha \to 0^+} \alpha^N \left( 1 - \tilde{\A}_{kk}^\dagger\left(\bnu_\infty(\alpha)\right) \right)^{-\frac{N}{N-2}} \neq 0
\implies
\lim_{\alpha \to 0^+} \tilde{\A}_{kk}^\dagger\left(\bnu_\infty(\alpha)\right) = 1
\text{\,,}
\]
from which we conclude:
\be
\inprod{I_d - \tilde{\A}^\dagger\left(\bnu_\infty(\alpha)\right)}{\tilde{W}^*_\deep} = \sum_{k=1}^d \big( 1 - \tilde{\A}_{kk}^\dagger \left(\bnu_\infty(\alpha)\right) \big) \cdot \big( \tilde{W}^*_\deep \big)_{kk} \xrightarrow[\alpha \to 0^+]{} 0
\label{eq:diag_comp_slack}
\text{\,.}
\ee
Returning to the original variables (un-diagonalizing by the orthogonal matrix~$O$), recall that:
\be
W^*_\deep = O^\top \, \tilde{W}^*_\deep \, O
\text{\,,}
\label{eq:W_final_undiag}
\ee
and:
\be
\A^\dagger\left(\bnu_\infty(\alpha)\right) = O^\top \tilde{\A}^\dagger\left(\bnu_\infty(\alpha)\right) O
\text{\,,}
\label{eq:A_adjoint_nu_undiag}
\ee
therefore the following hold:
\vspace{-1mm}
\begin{itemize}
\item $W^*_\deep \succeq 0$ ~~(by Equations~\eqref{eq:W_final_diag_psd} and~\eqref{eq:W_final_undiag});
\item $\A(W^*_\deep) = \y$ ~~(by assumption);
\item $\forall \alpha >0 : \A^\dagger\left(\bnu_\infty(\alpha)\right) \prec I_d$ ~~(by Equations~\eqref{eq:nu_admissible} and~\eqref{eq:A_adjoint_nu_undiag}); and
\item $\lim_{\alpha \to 0^+} \big\langle I_d - \tilde{\A}^\dagger\left(\bnu_\infty(\alpha)\right) , \tilde{W}^*_\deep \big\rangle = 0$ ~~(by Equations~\eqref{eq:diag_comp_slack}, \eqref{eq:W_final_undiag} and~\eqref{eq:A_adjoint_nu_undiag}).
\end{itemize}
\vspace{-1.5mm}
Lemma~\ref{lem:nuclear-opt-condition} below then concludes the proof.
\qed

\medskip

\begin{lemma} \label{lem:nuclear-opt-condition}
Suppose that $W^* \in \S_+^d$ satisfies $\A(W^*) = \y$, and that there exists a sequence of vectors $\bnu_1, \bnu_2, \ldots \in \R^m$ such that $\A^\dagger(\bnu_n) \preceq I$ for all~$n$ and $\lim_{n\to\infty} \langle I - \A^\dagger(\bnu_n), W^* \rangle = 0$.
Then $W^* \in \argmin_{W\in\S_+^d,\,\A(W)=\y}\norm{W}_*$.
\end{lemma}
\begin{proof}
Recall that~$\S_+^d$ stands for the set of (symmetric and) positive semidefinite matrices in~$\R^{d, d}$, and $\| \cdot \|_*$ denotes matrix nuclear norm.
The minimization problem being considered can be framed as a semidefinite program:\note{
Note that for $W \in \S_+^d$ we have $\norm{W}_* = \langle I, W \rangle$.
}
\be \label{eqn:nuclear_minimization_sdp}
\begin{aligned}
\text{minimize} \quad & \langle I, W \rangle \\
\text{subject to} \quad & \A(W) = \y \\
& W \in \S_+^d \,.
\end{aligned}
\ee
A corresponding dual program is:
\be \label{eqn:nuclear_minimization_sdp_dual}
\begin{aligned}
\text{maximize} \quad &  \bnu^\top\y \\
\text{subject to} \quad & \A^*(\bnu) \preceq I \\
& \bnu \in \R^m \,.
\end{aligned}
\ee
Let $\opt$ be the optimal value for the primal program (Equation~\eqref{eqn:nuclear_minimization_sdp}):
\[
\opt := \min\nolimits_{W\in\S_+^d,\,\A(W)=\y}\|W\|_*
\text{\,.}
\]
By duality theory, for any $\bnu$ feasible in the dual program (Equation~\eqref{eqn:nuclear_minimization_sdp_dual}), we have $\bnu^\top\y \le \opt$.
Since~$W^*$ is feasible in the primal, and each~$\bnu_n$ is feasible in the dual, it holds that:
\begin{align*}
    0 \le \norm{W^*}_* - \opt
    \le \norm{W^*}_* - \bnu_n^\top\y
    = \langle I, W^* \rangle - \bnu_n^\top\A(W^*)
    = \langle I - \A^\dagger(\bnu_n), W^* \rangle \,.
\end{align*}
Taking the limit $n\to\infty$, the right hand side above becomes~$0$, which implies $\norm{W^*}_* = \opt$.
\end{proof}

\subsection{Proof of Proposition~\ref{prop:schatten_disq}} \label{app:proofs:schatten_disq}

We will choose $A_1, \ldots, A_m$ to be diagonal.
This of course ensures symmetry and commutativity.
Additionally, by the proof of Theorem~\ref{thm:nuclear_dmf} (Appendix~\ref{app:proofs:nuclear_dmf}), it implies that $\bar{W}_\deep$ is diagonal and positive semidefinite.\note{
In the proof of Theorem~\ref{thm:nuclear_dmf} (Appendix~\ref{app:proofs:nuclear_dmf}), diagonality of $A_1, \ldots, A_m$ corresponds to the case where~$O$~---~the diagonalizing matrix~---~is simply the identity, and therefore $\bar{W}_\deep$ is equal to~$\tilde{W}^*_\deep$, implying that the former is indeed diagonal and positive semidefinite.
}
We set $A_1, \ldots, A_m$ and $y_1, \ldots, y_m$ such that the linear equations $\inprod{A_i}{W} = y_i$, $i = 1, \ldots, m$, are the following:
\be
\begin{aligned}
& W_{11} = W_{22} \\[1mm]
& W_{11} = W_{kk} + 1 \quad,~ k = 3, 4, \ldots, d
\text{\,.}
\end{aligned}
\label{eq:schatten_disq_constr}
\ee
Note that the matrices $A_1, \ldots, A_m$ which naturally induce these equations are (diagonal and) linearly independent, as required.
We know that $\bar{W}_\deep$ is diagonal and has minimal nuclear norm among all positive semidefinite matrices that satisfy the equations.
Using the fact that for positive semidefinite matrices nuclear norm is the same as trace, one readily sees that:
\[
\bar{W}_\deep = \diag(1, 1, 0, 0, \ldots, 0)
\text{\,.}
\]

We complete the proof by showing that in any neighborhood of~$\bar{W}_\deep$, there exists a positive semidefinite matrix that meets Equation~\eqref{eq:schatten_disq_constr} and has strictly smaller Schatten-$p$ quasi-norm for any $0 < p < 1$.
Indeed, for $\epsilon \in (0, 1)$ define:
\begin{align*}
\hat{W}_\epsilon := 
\begin{pmatrix}
1 & \epsilon & 0 & \cdots & 0 \\
\epsilon & 1 & 0 & \cdots & 0\\
0 & 0 & 0 & \cdots & 0\\
\vdots & \vdots & \vdots & \ddots & \vdots \\
0 & 0 & 0 & \cdots & 0
\end{pmatrix} 
\in \R^{d,d} 
\text{\,.}
\end{align*}
$\hat{W}_\epsilon$~obviously satisfies Equation~\eqref{eq:schatten_disq_constr}.
Additionally, it is symmetric with eigenvalues:
\[
\lambda_1 = 1 + \epsilon ~~,~~ \lambda_2 = 1 - \epsilon ~~,~~ \lambda_3 = \cdots = \lambda_d = 0
\text{\,,}
\]
and therefore is positive semidefinite.
For any $0 < p < 1$:
\[ 
\| \hat{W}_\epsilon \|_{S_p}^p = (1-\epsilon)^p + (1+\epsilon)^p < 2 \cdot \left(\tfrac{1}{2}(1 + \epsilon) + \tfrac{1}{2}(1 - \epsilon) \right)^p = 2 = \| \bar{W}_\deep \|_{S_p}^p
\text{\,,}
\]
where the inequality follows from $\theta_p: \R_{\geq 0} \to \R_{\geq 0}$, $\theta_p(x) = x^p$, being strictly concave.
Noting that taking $\epsilon \to 0^+$ makes $\hat{W}_\epsilon$ arbitrarily close to~$\bar{W}_\deep$, we conclude the proof.
\qed

\subsection{Proof of Lemma~\ref{lemma:asvd}} \label{app:proofs:asvd}

By Theorem~1 in~\cite{bunse1991numerical}, it suffices to show that the product matrix~$W(t)$ is an analytic function of~$t$.
Analytic functions are closed under summation, multiplication and composition, so the analyticity of~$\ell(\cdot)$ implies that~$\phi(\cdot)$ (Equation~\eqref{eq:phi}) is analytic as well.
It then follows (see Theorem~1.1 in~\cite{ilyashenko2008lectures}) that under gradient flow (Equation~\eqref{eq:gf}), the factors $W_1(t), \ldots, W_N(t)$ are analytic functions of~$t$.
Therefore~$W(t)$ (Equation~\eqref{eq:dmf}) is also analytic in~$t$.
\qed

\subsection{Proof of Theorem~\ref{thm:sing_vals_evolve}} \label{app:proofs:sing_vals_evolve}

Differentiate the analytic singular value decomposition (Equation~\eqref{eq:asvd}) with respect to time:
\[
\dot{W}(t) = \dot{U}(t) S(t) V^\top(t) + U(t) \dot{S}(t) V^\top(t) + U(t) S(t) \dot{V}^\top(t)
\text{\,,}
\]
then multiply from the left by~$U^\top(t)$ and from the right by~$V(t)$:
\[
U^\top(t) \dot{W}(t) V(t) = U^\top(t) \dot{U}(t) S(t) + \dot{S}(t) + S(t) \dot{V}^\top(t) V(t)
\text{\,,}
\]
where we used the fact that $U(t)$ and~$V(t)$ have orthonormal columns.
Restricting our attention to the diagonal elements of this matrix equation, we have:
\[
\uu_r^\top(t) \dot{W}(t) \vv_r(t) = \inprod{\uu_r(t)}{\dot{\uu}_r(t)} \sigma_r(t) + \dot{\sigma}_r(t) + \sigma_r(t) \inprod{\dot{\vv}_r(t)}{\vv_r(t)}
\quad,~ r = 1, \ldots, \min\{d, d'\}
\text{\,.}
\]
Since $\uu_r(t)$ has constant (unit) length it holds that $\inprod{\uu_r(t)}{\dot{\uu}_r(t)} = \frac{1}{2} \cdot \frac{d}{dt} \norm{\uu_r(t)}_2^2 = 0$, and similarly $\inprod{\dot{\vv}_r(t)}{\vv_r(t)} = 0$.
The latter equation thus simplifies to:
\be
\dot{\sigma}_r(t) = \uu_r^\top(t) \dot{W}(t) \vv_r(t)
\quad,~ r = 1, \ldots, \min\{d, d'\}
\text{\,.}
\label{eq:S_evolve_impl}
\ee
Lemma~\ref{lem:end_to_end_dynamics} from Appendix~\ref{app:lemmas} provides the following expression for~$\dot{W}(t)$:
\[
\dot{W}(t) = -\sum\nolimits_{j = 1}^N \left[ W(t) W^\top(t) \right]^\frac{j - 1}{N} \cdot
\nabla \ell\big(W(t)\big) \cdot \left[ W^\top(t) W(t) \right]^\frac{N - j}{N}
\text{\,,}
\]
where~$[\,\cdot\,]^\alpha$, $\alpha \in \R_{\geq 0}$, stands for a power operator defined over positive semidefinite matrices (with $\alpha = 0$ yielding identity by definition).
Plugging in the analytic singular value decomposition (Equation~\eqref{eq:asvd}) gives:
\beas
\dot{W}(t) 
&=& - \nabla \ell\big(W(t)\big) \cdot V(t) \big( S^2(t) \big)^\frac{N - 1}{N} V^\top(t) 
\\[0.5mm]
&& - \sum\nolimits_{j = 2}^{N - 1} U(t) \big( S^2(t) \big)^\frac{j - 1}{N} U^\top(t) \cdot \nabla \ell\big(W(t)\big) \cdot V(t) \big( S^2(t) \big)^\frac{N - j}{N} V^\top(t) 
\\
&& - U(t) \big( S^2(t) \big)^\frac{N - 1}{N} U^\top(t) \cdot \nabla \ell\big(W(t)\big)
\text{\,.}
\eeas
Left-multiplying by~$\uu_r^\top(t)$, right-multiplying by~$\vv_r(t)$, and using the fact that $\{\uu_r(t)\}_r$ (columns of~$U(t)$) and~$\{\vv_r(t)\}_r$ (columns of~$V(t)$) are orthonormal sets, we obtain:
\beas
\uu_r^\top(t) \dot{W}(t) \vv_r(t)
&=& - \uu_r^\top(t) \nabla \ell\big(W(t)\big) \vv_r(t) \cdot ( \sigma_r^2(t) )^\frac{N - 1}{N}
\\[0.5mm]
&& - \sum\nolimits_{j = 2}^{N - 1} ( \sigma_r^2(t) )^\frac{j - 1}{N} \cdot \uu_r^\top(t) \nabla \ell\big(W(t)\big) \vv_r(t) \cdot ( \sigma_r^2(t) )^\frac{N - j}{N}
\\
&& - ( \sigma_r^2(t) )^\frac{N - 1}{N} \cdot \uu_r^\top(t) \nabla \ell\big(W(t)\big) \vv_r(t)
\\[1mm]
&=& - N \cdot ( \sigma_r^2(t) )^\frac{N - 1}{N} \cdot \uu_r^\top(t) \nabla \ell\big(W(t)\big) \vv_r(t)
\text{\,.}
\eeas
Combining this with Equation~\eqref{eq:S_evolve_impl} yields the sought-after Equation~\eqref{eq:S_evolve}.

\medskip

To complete the proof, it remains to show that if the matrix factorization is non-degenerate (has depth $N \geq 2$), singular values need not be signed, \ie~we may assume $\sigma_r(t) \geq 0$ for all~$t$.
Equation~\eqref{eq:S_evolve}, along with Lemma~\ref{lemma:sign_preserve}, imply that if $N \geq 2$, $\sigma_r(t)$~will never switch sign.
Therefore, either $\sigma_r(t) \geq 0$ for all~$t$, or alternatively, this will hold if we take away a minus sign from~$\sigma_r(t)$ and absorb it into~$\uu_r(t)$ (or~$\vv_r(t)$).
\qed

\subsection{Proof of Lemma~\ref{lemma:sing_vecs_evolve}} \label{app:proofs:sing_vecs_evolve}

A real analytic function is either identically zero, or admits a zero set with no accumulation points (\cf~\cite{krantz2002primer}).
For any $r \in \{ 1, \ldots, \min\{d, d'\} \}$, applying this fact to the signed singular value~$\sigma_r(t)$, while taking into account our assumption of it being different from zero at initialization, we conclude that the set of times~$t$ for which it vanishes has no accumulation points.
Similarly, for any $r, r' \in \{ 1, \ldots, \min\{d, d'\} \}$, $r \neq r'$, we assumed that $\sigma_r^2(t) - \sigma_{r'}^2(t)$ is different from zero at initialization, and thus the set of times~$t$ for which it vanishes is free from accumulation points.
Overall, any time~$t$ for which $\sigma_r(t) = 0$ for some~$r$, or $\sigma_r^2(t) = \sigma_{r'}^2(t)$ for some~$r \neq r'$, must be isolated, \ie~surrounded by a neighborhood in which none of these conditions are met.
Accordingly, hereafter, we assume $\forall r : \sigma_r(t) \neq 0$ and $\forall r \neq r' : \sigma_r^2(t) \neq \sigma_{r'}^2(t)$, knowing that for times~$t$ in which this does not hold, $\dot{U}(t)$ and~$\dot{V}(t)$ can be inferred by continuity.

\medskip

We now follow a series of steps adopted from~\cite{townsend2016differentiating}, to derive expressions for $\dot{U}(t)$ and~$\dot{V}(t)$ in terms of $U(t)$, $V(t)$, $S(t)$ and~$\dot{W}(t)$.
Differentiate the analytic singular value decomposition (Equation~\eqref{eq:asvd}) with respect to time:
\be
\dot{W}(t) = \dot{U}(t) S(t) V^\top(t) + U(t) \dot{S}(t) V^\top(t) + U(t) S(t) \dot{V}^\top(t)
\text{\,.}
\label{eq:asvd_deriv}
\ee
Multiplying from the left by~$U^\top(t)$ and from the right by~$V(t)$, we have:
\be
U^\top(t) \dot{W}(t) V(t) = U^\top(t) \dot{U}(t) S(t) + \dot{S}(t) + S(t) \dot{V}^\top(t) V(t)
\text{\,,}
\label{eq:asvd_deriv_mult}
\ee
where we used the fact that~$U(t)$ and~$V(t)$ have orthonormal columns.
This orthonormality also implies that $U^\top(t) \dot{U}(t)$ and~$\dot{V}^\top(t) V(t)$ are skew-symmetric,\note{
To see this, note that $U^\top(t) U(t)$ is constant, thus its derivative with respect to time is equal to zero, \ie~$\dot{U}^\top(t) U(t) + U^\top(t) \dot{U}(t) = 0$ (by an analogous argument $\dot{V}^\top(t) V(t) + V^\top(t) \dot{V}(t) = 0$ holds as well).
}
and in particular have zero diagonals.
Since~$S(t)$ is diagonal, $U^\top(t) \dot{U}(t) S(t)$ and $S(t) \dot{V}^\top(t) V(t)$ have zero diagonals as well.
On the other hand $\dot{S}(t)$ holds zeros outside its diagonal, and so we may write:
\be
\bar{I}_{\min\{d, d'\}} \odot (U^\top(t) \dot{W}(t) V(t)) = U^\top(t) \dot{U}(t) S(t) + S(t) \dot{V}^\top(t) V(t)
\text{\,,}
\label{eq:asvd_deriv_mult_offdiag}
\ee
where $\odot$~stands for Hadamard (element-wise) product, and $\bar{I}_{\min\{d, d'\}}$ is a $\min\{d, d'\} \times \min\{d, d'\}$ matrix holding zeros on its diagonal and ones elsewhere.
Taking transpose of Equation~\eqref{eq:asvd_deriv_mult_offdiag}, while recalling that $U^\top(t) \dot{U}(t)$ and~$\dot{V}^\top(t) V(t)$ are skew-symmetric, we have:
\be
\bar{I}_{\min\{d, d'\}} \odot (V^\top(t) \dot{W}^\top(t) U(t)) = - S(t) U^\top(t) \dot{U}(t) - \dot{V}^\top(t) V(t) S(t)
\text{\,.}
\label{eq:asvd_deriv_mult_offdiag_tr}
\ee
Right-multiply Equation~\eqref{eq:asvd_deriv_mult_offdiag} by~$S(t)$, left-multiply Equation~\eqref{eq:asvd_deriv_mult_offdiag_tr} by~$S(t)$, and add:
\[
\bar{I}_{\min\{d, d'\}} \odot (U^\top(t) \dot{W}(t) V(t) S(t) + S(t) V^\top(t) \dot{W}^\top(t) U(t)) = U^\top(t) \dot{U}(t) S^2(t) - S^2(t) U^\top(t) \dot{U}(t)
\text{\,.}
\]
Since we assume diagonal elements of~$S^2(t)$ are distinct ($\sigma_r^2(t) \neq \sigma_{r'}^2(t)$ for $r \neq r'$), this implies:
\[
U^\top(t) \dot{U}(t) = H(t) \odot \big[ U^\top(t) \dot{W}(t) V(t) S(t) + S(t) V^\top(t) \dot{W}^\top(t) U(t) \big]
\text{\,,}
\]
where the matrix $H(t) \in \R^{\min\{d, d'\}, \min\{d, d'\}}$ is defined by:
\be
H_{r, r'}(t) :=
\begin{cases}
\big( \sigma_{r'}^2(t) - \sigma_r^2(t) \big)^{-1} & , r \neq r' \\
~\quad\qquad 0 & , r = r'
\end{cases}
\text{\,.}
\label{eq:H}
\ee
Multiplying from the left by~$U(t)$ yields:
\be
P_{U(t)} \dot{U}(t) = U(t) \big( H(t) \odot \big[ U^\top(t) \dot{W}(t) V(t) S(t) + S(t) V^\top(t) \dot{W}^\top(t) U(t) \big] \big)
\text{\,,}
\label{eq:U_evolve_impl_proj}
\ee
with $P_{U(t)} := U(t) U^\top(t)$ being the projection onto the subspace spanned by the (orthonormal) columns of~$U(t)$.
Denote by~$P_{U_\perp(t)}$ the projection onto the orthogonal complement, \ie~$P_{U_\perp(t)} := I_d - U(t) U^\top(t)$, where $I_d$ is the $d \times d$ identity matrix.
Apply~$P_{U_\perp(t)}$ to both sides of Equation~\eqref{eq:asvd_deriv}:
\[
P_{U_\perp(t)} \dot{W}(t) = P_{U_\perp(t)} \dot{U}(t) S(t) V^\top(t) + P_{U_\perp(t)} U(t) \dot{S}(t) V^\top(t) + P_{U_\perp(t)} U(t) S(t) \dot{V}^\top(t)
\text{\,.}
\]
Note that $P_{U_\perp(t)} U(t) = 0$, and multiply from the right by $V(t) S^{-1}(t)$ (the latter is well-defined since we assume diagonal elements of~$S(t)$ are non-zero~---~$\sigma_r(t) \neq 0$):
\be
P_{U_\perp(t)} \dot{U}(t) = P_{U_\perp(t)} \dot{W}(t) V(t) S^{-1}(t) = \big( I_d - U(t) U^\top(t) \big) \dot{W}(t) V(t) S^{-1}(t)
\text{\,.}
\label{eq:U_evolve_impl_perp_proj}
\ee
Adding Equations~\eqref{eq:U_evolve_impl_proj} and~\eqref{eq:U_evolve_impl_perp_proj}, we obtain an expression for~$\dot{U}(t)$:
\bea
\dot{U}(t) &=& P_{U(t)} \dot{U}(t) + P_{U_\perp(t)} \dot{U}(t) 
\nonumber\\[0.5mm]
&=& U(t) \big( H(t) \odot \big[ U^\top(t) \dot{W}(t) V(t) S(t) + S(t) V^\top(t) \dot{W}^\top(t) U(t) \big] \big) 
\nonumber\\[0.5mm]
&&\quad + \big( I_d - U(t) U^\top(t) \big) \dot{W}(t) V(t) S^{-1}(t)
\text{\,.}
\label{eq:U_evolve_impl}
\eea
By returning to Equations~\eqref{eq:asvd_deriv_mult_offdiag} and~\eqref{eq:asvd_deriv_mult_offdiag_tr}, switching the directions from which they were multiplied by~$S(t)$ (\ie~multiplying Equation~\eqref{eq:asvd_deriv_mult_offdiag} from the left and Equation~\eqref{eq:asvd_deriv_mult_offdiag_tr} from the right), and continuing similarly to above, an analogous expression for~$\dot{V}(t)$ is derived: 
\bea
\dot{V}(t) &=& V(t) \big( H(t) \odot \big[S(t) U^\top(t) \dot{W}(t) V(t) + V^\top(t) \dot{W}^\top(t) U(t) S(t) \big] \big) 
\nonumber\\[0.5mm]
&&\quad + \big( I_{d'} - V(t) V^\top(t) \big) \dot{W}^\top(t) U(t) S^{-1}(t)
\text{\,,}
\label{eq:V_evolve_impl}
\eea
where $I_{d'}$ is the $d' \times d'$ identity matrix.

\medskip

Next, we invoke Lemma~\ref{lem:end_to_end_dynamics} from Appendix~\ref{app:lemmas}, which provides an expression for~$\dot{W}(t)$:
\be
\dot{W}(t) = -\sum\nolimits_{j = 1}^N \left[ W(t) W^\top(t) \right]^\frac{j - 1}{N} \cdot
\nabla \ell\big(W(t)\big) \cdot \left[ W^\top(t) W(t) \right]^\frac{N - j}{N}
\text{\,,}
\label{eq:W_evolve}
\ee
where~$[\,\cdot\,]^\alpha$, $\alpha \in \R_{\geq 0}$, stands for a power operator defined over positive semidefinite matrices (with $\alpha = 0$ yielding identity by definition).
Plug the analytic singular value decomposition (Equation~\eqref{eq:asvd}) into Equation~\eqref{eq:W_evolve}:
\bea
\dot{W}(t) 
&=& - \nabla \ell\big(W(t)\big) \cdot V(t) \big( S^2(t) \big)^\frac{N - 1}{N} V^\top(t) 
\nonumber\\[0.5mm]
&& - \sum\nolimits_{j = 2}^{N - 1} U(t) \big( S^2(t) \big)^\frac{j - 1}{N} U^\top(t) \cdot \nabla \ell\big(W(t)\big) \cdot V(t) \big( S^2(t) \big)^\frac{N - j}{N} V^\top(t) 
\nonumber\\
&& - U(t) \big( S^2(t) \big)^\frac{N - 1}{N} U^\top(t) \cdot \nabla \ell\big(W(t)\big)
\text{\,.}
\label{eq:W_evolve_asvd}
\eea
From this it follows that:
\bea
U^\top(t) \dot{W}(t) V(t) 
&=& - U^\top(t) \nabla \ell\big(W(t)\big) V(t) \big( S^2(t) \big)^\frac{N - 1}{N} 
\nonumber\\[0.5mm]
&& - \sum\nolimits_{j = 2}^{N - 1} \big( S^2(t) \big)^\frac{j - 1}{N} U^\top(t) \nabla \ell\big(W(t)\big) V(t) \big( S^2(t) \big)^\frac{N - j}{N} 
\nonumber\\
&& - \big( S^2(t) \big)^\frac{N - 1}{N} U^\top(t) \nabla \ell\big(W(t)\big) V(t) 
\nonumber\\[1mm]
&=& - G(t) \odot \left[ U^\top(t) \nabla \ell\big(W(t)\big) V(t) \right]
\text{\,,}
\label{eq:W_evolve_asvd_mult}
\eea
where $G(t) \in \R^{\min\{d, d'\}, \min\{d, d'\}}$ is defined by: 
\be
G_{r, r'}(t) := \sum\nolimits_{j = 1}^N (\sigma_r^2(t))^\frac{j - 1}{N} (\sigma_{r'}^2(t))^\frac{N - j}{N}
\text{\,.}
\label{eq:G}
\ee
Since $G(t)$ is symmetric (and $S(t)$ is diagonal), Equation~\eqref{eq:W_evolve_asvd_mult} implies:
\beas
&& U^\top(t) \dot{W}(t) V(t) S(t) + S(t) V^\top(t) \dot{W}^\top(t) U(t) \\
&&\qquad = - G(t) \odot \big[ U^\top(t) \nabla \ell\big(W(t)\big) V(t) S(t) + S(t) V^\top(t) \nabla \ell^\top\big(W(t)\big) U(t) \big]
\text{\,.}
\eeas
Taking Hadamard product by~$H(t)$ (Equation~\eqref{eq:H}) we obtain:
\beas
&& H(t) \odot \big[ U^\top(t) \dot{W}(t) V(t) S(t) + S(t) V^\top(t) \dot{W}^\top(t) U(t) \big] \\
&&\qquad = - F(t) \odot \big[ U^\top(t) \nabla \ell\big(W(t)\big) V(t) S(t) + S(t) V^\top(t) \nabla \ell^\top\big(W(t)\big) U(t) \big]
\text{\,,}
\eeas
where $F(t) := H(t) \odot G(t)$ is given by:
\be
F_{r, r'}(t) :=
\begin{cases}
\big( (\sigma_{r'}^2(t))^{1 / N} - (\sigma_r^2(t))^{1 / N} \big)^{-1} & , r \neq r' \\
\,~~\quad\qquad\qquad 0 & , r = r'
\end{cases}
\text{\,.}
\label{eq:F}
\ee
Plug this into Equation~\eqref{eq:U_evolve_impl}:
\bea
\dot{U}(t) 
&=& - U(t) \big( F(t) \odot \big[ U^\top(t) \nabla \ell\big(W(t)\big) V(t) S(t) + S(t) V^\top(t) \nabla \ell^\top\big(W(t)\big) U(t) \big] \big) 
\nonumber\\[0.5mm]
&&\quad + \big( I_d - U(t) U^\top(t) \big) \dot{W}(t) V(t) S^{-1}(t)
\text{\,.}
\label{eq:U_evolve_half_impl}
\eea
The first term on the right-hand side here complies with the result we seek to prove (Equation~\eqref{eq:U_evolve}).
To treat the second term, we again invoke Equation~\eqref{eq:W_evolve_asvd}, noting that the matrix $P_{U_\perp(t)} := I_d - U(t) U^\top(t)$ (projection onto the orthogonal complement of the subspace spanned by the columns of~$U(t)$) produces zero when right-multiplied by~$U(t)$.
This implies:
\[
\big( I_d - U(t) U^\top(t) \big) \dot{W}(t) V(t) S^{-1}(t) = - \big( I_d - U(t) U^\top(t) \big) \nabla \ell\big(W(t)\big) V(t) \big( S^2(t) \big)^{\frac{1}{2} - \frac{1}{N}}
\text{\,.}
\]
Plugging this back into Equation~\eqref{eq:U_evolve_half_impl} yields Equation~\eqref{eq:U_evolve}~---~sought-after result.
The analogous Equation~\eqref{eq:V_evolve} can be derived in a similar fashion (by incorporating Equation~\eqref{eq:W_evolve} into Equation~\eqref{eq:V_evolve_impl}, as we have done for Equation~\eqref{eq:U_evolve_impl}).
\qed

\subsection{Proof of Corollary~\ref{cor:sing_vecs_station}} \label{app:proofs:sing_vecs_station}

As stated in the proof of Lemma~\ref{lemma:sing_vecs_evolve} (Appendix~\ref{app:proofs:sing_vecs_evolve}), for all~$t$ but a set of isolated points it holds that $\forall r : \sigma_r(t) \neq 0$ and $\forall r \neq r' : \sigma_r^2(t) \neq \sigma_{r'}^2(t)$, meaning Equations~\eqref{eq:U_evolve} and~\eqref{eq:V_evolve} are well-defined.
We will initially assume this to be the case, and then treat isolated points by taking limits.
Left-multiply Equation~\eqref{eq:U_evolve} by~$U^\top(t)$ and Equation~\eqref{eq:V_evolve} by~$V^\top(t)$:
\beas
U^\top(t) \dot{U}(t) &=& - F(t) \odot \left[ U^\top(t) \nabla\ell(W(t)) V(t) S(t) + S(t) V^\top(t) \nabla\ell^\top(W(t)) U(t) \right] \\[0.5mm]
V^\top(t) \dot{V}(t) &=& - F(t) \odot \left[ S(t) U^\top(t) \nabla\ell(W(t)) V(t) + V^\top(t) \nabla\ell^\top(W(t)) U(t) S(t) \right]
\text{\,,}
\eeas
where we have used the fact that~$U(t)$ and~$V(t)$ have orthonormal columns.
Right-multiplying the first equation by~$S(t)$, left-multiplying the second by~$S(t)$, and then subtracting, we obtain:
\beas
&& U^\top(t) \dot{U}(t) S(t) - S(t) V^\top(t) \dot{V}(t) \\[0.5mm]
&&\qquad = - F(t) \odot \left[ U^\top(t) \nabla\ell(W(t)) V(t) S^2(t) - S^2(t) U^\top(t) \nabla\ell(W(t)) V(t) \right] \\[0.5mm]
&&\qquad = - F(t) \odot E(t) \odot \left[ U^\top(t) \nabla\ell(W(t)) V(t) \right]
\text{\,,}
\eeas
where the matrix $E(t) \in \R^{\min\{d, d'\}, \min\{d, d'\}}$ is defined by: $E_{r, r'}(t) := \sigma_{r'}^2(t) - \sigma_r^2(t)$.
Recalling the definition of~$F(t)$ (Equation~\eqref{eq:F}), we have:
\be
U^\top(t) \dot{U}(t) S(t) - S(t) V^\top(t) \dot{V}(t) = - \bar{I}_{\min\{d, d'\}} \odot G(t) \odot \left[ U^\top(t) \nabla\ell(W(t)) V(t) \right]
\text{\,,}
\label{eq:station_diag}
\ee
where $G(t) \in \R^{\min\{d, d'\}, \min\{d, d'\}}$ is the matrix defined in Equation~\eqref{eq:G}, and $\bar{I}_{\min\{d, d'\}}$ is a matrix of the same size, with zeros on its diagonal and ones elsewhere.
Since by assumption $\forall r : \sigma_r(t) \neq 0$, the matrix~$G(t)$ does not contain zero elements.
Therefore when $\dot{U}(t) = 0$ and~$\dot{V}(t) = 0$, leading the left-hand side of Equation~\eqref{eq:station_diag} to vanish, it must be that $U^\top(t) \nabla\ell(W(t)) V(t)$ is diagonal.

\medskip

To complete the proof, it remains to treat those isolated times~$t$ for which the conditions $\forall r : \sigma_r(t) \neq 0$ and $\forall r \neq r' : \sigma_r^2(t) \neq \sigma_{r'}^2(t)$ do not all hold, and thus our derivation of Equation~\eqref{eq:station_diag} may be invalid.
Since both sides of the equation are continuous, it carries over to such isolated times, and is in fact applicable to any~$t$.
Accordingly, any~$t$ for which $\dot{U}(t) = 0$ and~$\dot{V}(t) = 0$ admits $\bar{I}_{\min\{d, d'\}} \odot G(t) \odot \left[ U^\top(t) \nabla\ell(W(t)) V(t) \right] = 0$.
Recalling the definition of~$G(t)$ (Equation~\eqref{eq:G}), it is clear that the latter equality implies diagonality of $U^\top(t) \nabla\ell(W(t)) V(t)$ if $\forall r : \sigma_r(t) \neq 0$.
This means that the sought-after result holds if $\forall r : \sigma_r(t) \neq 0$ for every~$t$.
Recollecting that the product matrix is initialized to be full-rank ($\forall r : \sigma_r(0) \neq 0$), and invoking our assumption on the factorization being non-degenerate ($N \geq 2$), we apply Lemma~\ref{lemma:sign_preserve} (from Appendix~\ref{app:proofs:sing_vals_evolve}) to the evolution of~$\{\sigma_r(t)\}_r$ (Equation~\eqref{eq:S_evolve} in Theorem~\ref{thm:sing_vals_evolve}) and conclude the proof.
\qed

\section{Extension of~\cite{gunasekar2017implicit} to asymmetric matrix factorization} \label{app:gunasekar_asym}

Extending Theorem~\ref{thm:nuclear_mf} from~\cite{gunasekar2017implicit} to asymmetric (depth-$2$) matrix factorizations boils down to proving the following proposition:
\begin{proposition} \label{prop:nuclear_mf_asym}
Consider gradient flow on the objective:
\[
\phi(W_1, W_2) = \ell(W_2 W_1) = \frac{1}{2} \sum\nolimits_{i = 1}^m (y_i - \inprod{A_i}{W_2 W_1})^2
\text{\,,}
\]
with $W_1, W_2 \in \R^{d, d}$ initialized to~$\alpha I$, $\alpha > 0$, and denote by~$W_{\shallow, \infty}(\alpha)$ the product matrix obtained at the end of optimization (\ie~$W_{\shallow, \infty}(\alpha) := \lim_{t \to \infty} W_2(t) W_1(t)$ where $W_j(0) = \alpha I$ and $\dot{W_j}(t) = -\frac{\partial \phi}{\partial W_j}(W_1(t), W_2(t))$ for~$t \in \R_{\geq 0}$).
Assume the measurement matrices $A_1, \ldots, A_m$ commute.
Then, if $\bar{W}_\shallow := \lim_{\alpha \to 0} W_{\shallow, \infty}(\alpha)$ exists and is a global optimum for Equation~\eqref{eq:psd_recover} with~$\ell(\bar{W}_\shallow) = 0$, it holds that $\bar{W}_\shallow \in \argmin_{W \in \S_+^d,\,\ell(W) = 0} \norm{W}_*$, \ie~$\bar{W}_\shallow$ is a global optimum with minimal nuclear~norm.
\end{proposition}
\begin{proof}
We follow the proof of Theorem~\ref{thm:nuclear_dmf} (Appendix~\ref{app:proofs:nuclear_dmf}) up until Equation~\eqref{eqn:e2e_diagonals_dynamics_in_norm_proof}.
Equation~\eqref{eqn:e2e_dynamics_in_norm_proof}, specialized to $N = 2$, yields dynamics for the product matrix $W(t) = W_2(t) W_1 (t)$:
\be \label{eqn:e2e_dynamics_in_two_layer_norm_proof}
\begin{aligned}
    \dot{W}(t) 
    &= - \A^*(\rr(t)) \cdot \left[ W^\top(t) W(t) \right]^\frac{1}{2} - \left[ W(t) W^\top(t) \right]^\frac{1}{2} \cdot \A^*(\rr(t)) \,, \\[1mm]
    W(0) &= \alpha^2 I \,.
\end{aligned}
\ee
Equation~\eqref{eqn:e2e_diagonals_dynamics_in_norm_proof}, along with Lemma~\ref{lemma:sign_preserve}, imply that $\tilde{W}_{kk}(t)$~maintains the sign of its initialization, \ie~is positive.
The diagonal matrix~$\tilde{W}(t)$ is therefore positive definite, and so is the product matrix $W(t) = O^\top \tilde{W}(t) O$.
Equation~\eqref{eqn:e2e_dynamics_in_two_layer_norm_proof} thus becomes:
\be \label{eqn:e2e_dynamics_in_two_layer_norm_proof_2}
\begin{aligned}
    \dot{W}(t) 
    &= - \A^*(\rr(t)) \cdot W(t) -  W(t)\cdot \A^*(\rr(t)) \,, \\[1mm]
    W(0) &= \alpha^2 I \,.
\end{aligned}
\ee
The dynamics in Equation~\eqref{eqn:e2e_dynamics_in_two_layer_norm_proof_2} are precisely those developed in~\cite{gunasekar2017implicit} for a symmetric matrix factorization.
The proof of Theorem~1 there can now be applied as is, establishing the desired result.
\end{proof}

\section{Further experiments and implementation details} \label{app:exper}

\subsection{Further experiments} \label{app:exper:further}

Figures~\ref{fig:exper_intro_supp}, \ref{fig:exper_norm_supp} and~\ref{fig:exper_dyn_supp} present matrix sensing experiments supplementing the matrix completion experiments reported in Figures~\ref{fig:exper_intro}, \ref{fig:exper_norm} and~\ref{fig:exper_dyn} respectively.

\begin{figure}
\begin{center}
\includegraphics[width=\textwidth]{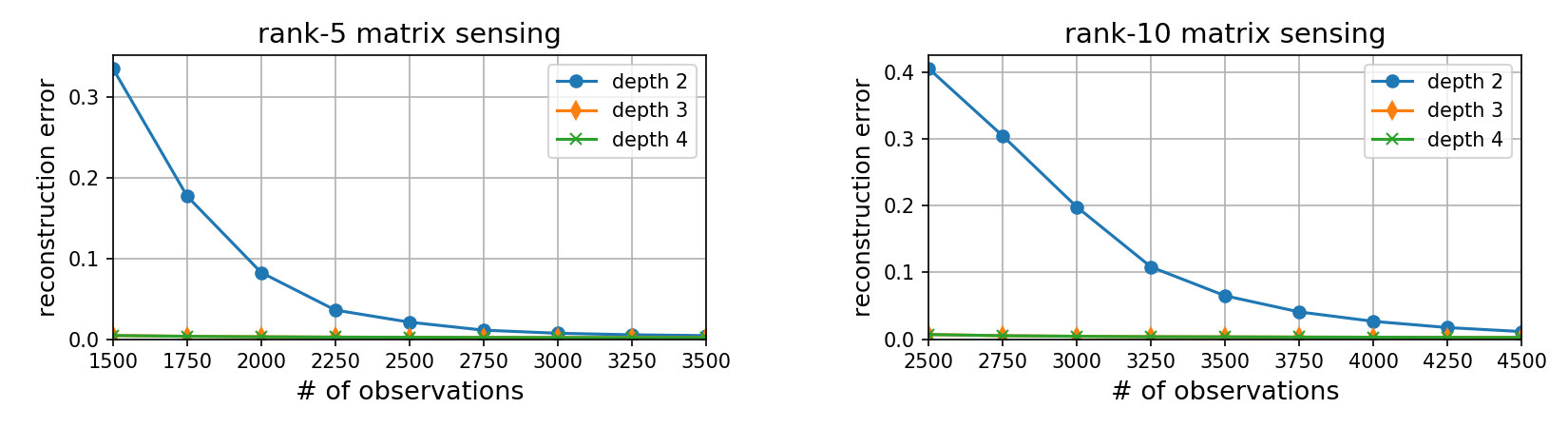}
\end{center}
\caption{
Matrix sensing via gradient descent over deep matrix factorizations.
This figure is identical to Figure~\ref{fig:exper_intro}, except that reconstruction of a ground truth matrix is based not on a randomly chosen subset of entries, but on a set of random projections (\ie~on $\{ \inprod{A_i}{W^*} \}_{i = 1}^m$ where $W^*$ is the ground truth and $A_1, \ldots, A_m$ are measurement matrices drawn independently from a Gaussian distribution).
For further details on this experiment~see~Appendix~\ref{app:exper:imple}.
}
\label{fig:exper_intro_supp}
\end{figure}

\begin{figure}
\begin{center}
\includegraphics[width=\textwidth]{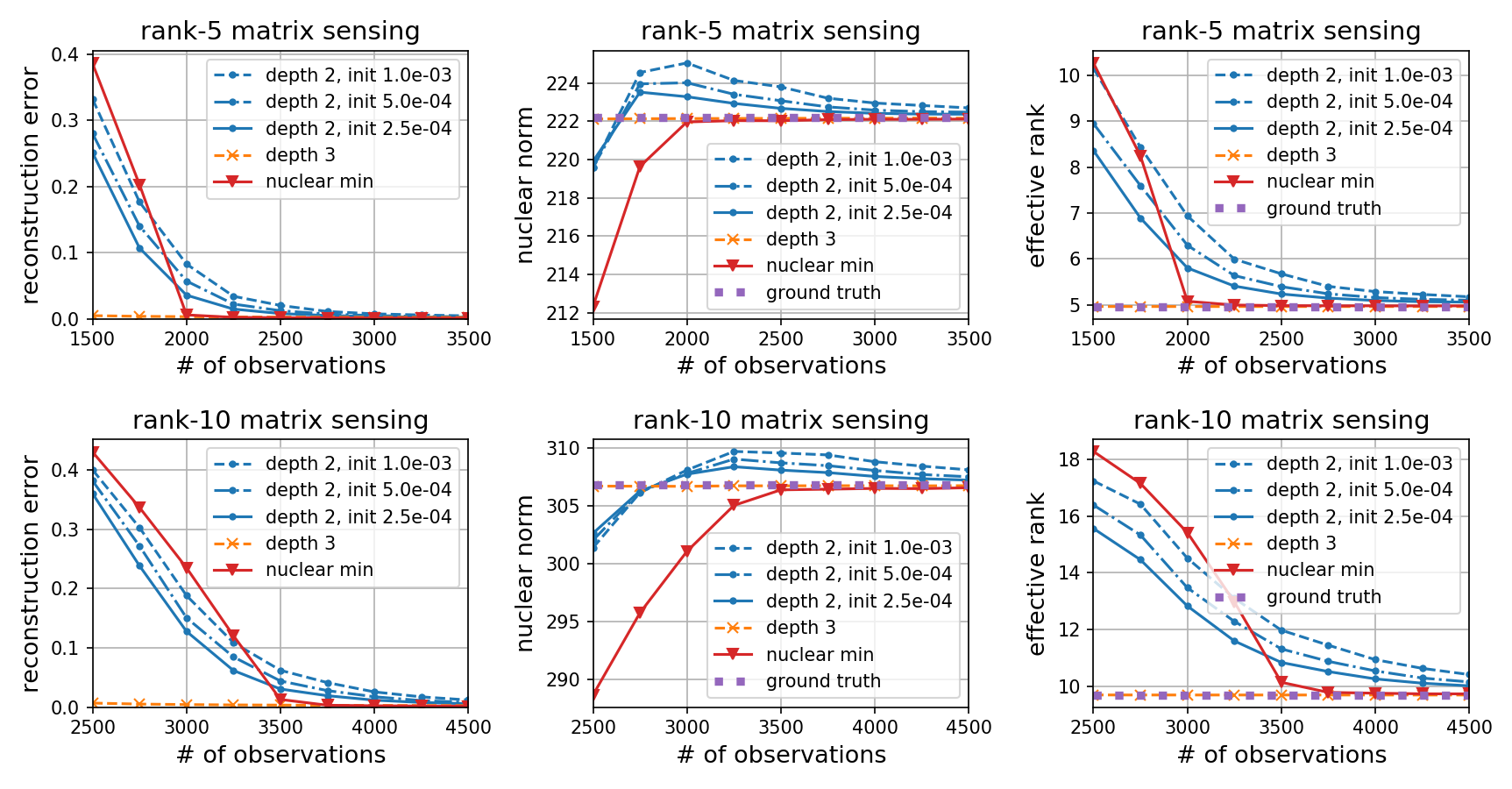}
\end{center}
\caption{
Evaluation of nuclear norm as the implicit regularization in deep matrix factorization on matrix sensing tasks.
This figure is identical to Figure~\ref{fig:exper_norm}, except that reconstruction of a ground truth matrix is based not on a randomly chosen subset of entries, but on a set of random projections (\ie~on $\{ \inprod{A_i}{W^*} \}_{i = 1}^m$ where $W^*$ is the ground truth and $A_1, \ldots, A_m$ are measurement matrices drawn independently from a Gaussian distribution).
For further details on this experiment see Appendix~\ref{app:exper:imple}.
}
\label{fig:exper_norm_supp}
\end{figure}

\begin{figure}
\begin{center}
\includegraphics[width=\textwidth]{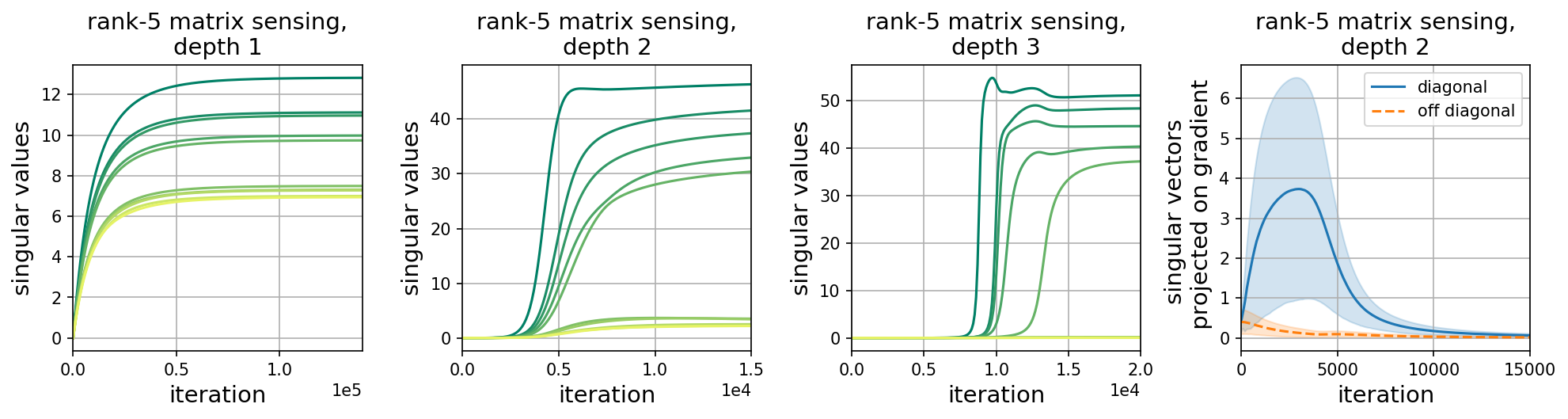}
\end{center}
\caption{
Dynamics of gradient descent over deep matrix factorizations on a matrix sensing task.
This figure is identical to the top row of Figure~\ref{fig:exper_dyn}, except that training is based not on $2000$ randomly chosen entries of the ground truth matrix, but on $2000$ random projections (\ie~on $\{ \inprod{A_i}{W^*} \}_{i = 1}^{2000}$ where $W^*$ is the ground truth and $A_1, \ldots, A_{2000}$ are measurement matrices drawn independently from a Gaussian distribution).
For further details on this experiment see Appendix~\ref{app:exper:imple}.
}
\label{fig:exper_dyn_supp}
\end{figure}

\subsection{Implementation details} \label{app:exper:imple}

In this appendix we provide implementation details omitted from the descriptions of our experiments (Figures~\ref{fig:exper_intro}, \ref{fig:exper_norm}, \ref{fig:exper_dyn}, \ref{fig:exper_intro_supp}, \ref{fig:exper_norm_supp} and~\ref{fig:exper_dyn_supp}).
Our implementation is based on Python, with PyTorch~(\cite{paszke2017automatic}) for realizing deep matrix factorizations and CVXPY~(\cite{diamond2016cvxpy,agrawal2018rewriting}) for finding minimum nuclear norm solutions.
Source code for reproducing our results can be found in \url{https://github.com/roosephu/deep_matrix_factorization}.

When referring to a random rank-$r$ matrix with size $d \times d'$, we mean a product~$U V^\top$, where the entries of $U \in \R^{d, r}$ and~$V \in \R^{d', r}$ are drawn independently from the standard normal distribution.
Randomly chosen observed entries in synthetic matrix completion tasks (Figures~\ref{fig:exper_intro}, \ref{fig:exper_norm} and top row of Figure~\ref{fig:exper_dyn}) were selected uniformly (without repetition).
In synthetic matrix sensing tasks (Figures~\ref{fig:exper_intro_supp}, \ref{fig:exper_norm_supp} and~\ref{fig:exper_dyn_supp}), entries of all measurement (projection) matrices were drawn independently from the standard normal distribution.
When varying the number of observations in synthetic matrix completion and sensing (Figures~\ref{fig:exper_intro}, \ref{fig:exper_norm}, \ref{fig:exper_intro_supp} and~\ref{fig:exper_norm_supp}), we evaluated increments of~$250$.
Training on MovieLens~100K dataset (bottom row of Figure~\ref{fig:exper_dyn}) comprised fitting $10000$ randomly (uniformly) chosen samples from the $100000$ entries given in the $943 \times 1682$ user-movie~rating~matrix~(see~\cite{harper2016movielens}).

In all experiments, deep matrix factorizations were trained by (full batch) gradient descent applied to $\ell_2$~loss over the observed entries (in matrix completion tasks) or given projections (in matrix sensing tasks), with no explicit regularization.
Gradient descent was initialized by independently sampling all weights from a Gaussian distribution with zero mean and configurable standard deviation.
Learning rates were fixed throughout optimization, and the stopping criterion was training loss reaching value lower than $10^{-6}$ (or $10^6$~iterations elapsing).
In the nuclear norm evaluation experiments (Figures~\ref{fig:exper_norm} and~\ref{fig:exper_norm_supp}), learning rate and standard deviation of initialization for gradient descent were assigned values from the set $\{10^{-3}, 5 \cdot 10^{-4}, 2.5 \cdot 10^{-4}\}$.
In the dynamics illustration experiments (Figures~\ref{fig:exper_dyn} and~\ref{fig:exper_dyn_supp}), displayed results correspond to both learning rate and standard deviation for initialization being~$10^{-3}$.

In figures~\ref{fig:exper_intro}, \ref{fig:exper_norm}, \ref{fig:exper_intro_supp} and~\ref{fig:exper_norm_supp}, each error bar marks standard deviation of the respective result over three trials differing in random seed for initialization of gradient descent.
Reconstruction error with respect to a ground truth matrix~$W^*$ is based on normalized Frobenius distance, \ie~for a solution~$W$ it is $\norm{W - W^*}_F / \norm{W^*}_F$.
In experiments with matrix completion and sensing under varying number of observations (Figures~\ref{fig:exper_intro}, \ref{fig:exper_norm}, \ref{fig:exper_intro_supp} and~\ref{fig:exper_norm_supp}), plots begin at the smallest number for which stable results were obtained, and end when all evaluated methods are close to zero reconstruction error.
For the dynamics illustration experiments (Figures~\ref{fig:exper_dyn} and~\ref{fig:exper_dyn_supp}), plots showing singular values hold $10$ curves corresponding to the largest ones.

\end{document}